\documentclass[11pt]{article}

\usepackage{graphicx}
\usepackage{graphicx}
\graphicspath{{Figures/}}
\usepackage{subcaption}
\usepackage{tabulary}
\usepackage{multirow}
\usepackage[margin=1in]{geometry}
\usepackage{titlesec}
\usepackage[labelfont=bf,textfont=it,font=footnotesize]{caption}
\usepackage{parskip}
\usepackage[comma,sort&compress,numbers]{natbib}
\usepackage{mathpazo}
\usepackage{ulem}
\usepackage[algo2e,ruled,noend]{algorithm2e}
\usepackage[utf8]{inputenc} 
\usepackage[T1]{fontenc}    
\usepackage{url}            
\usepackage{booktabs}       
\usepackage{microtype}      
\usepackage{xcolor}         
\usepackage{enumitem}
\usepackage{tikz}
\usepackage{pgfplots}
\usepackage{pgfplotstable}
\usepackage{multirow}
\usepackage[para,online,flushleft]{threeparttable}
\usepackage[labelfont=bf,textfont=it]{caption}
\usepackage{subcaption}
\usepackage{titlesec}
\usetikzlibrary{spy}

\usepackage{tabulary}

\usepackage{amssymb}
\usepackage{amsthm}
\usepackage{amsmath}

\titlespacing\section{0pt}{6pt plus 2pt minus 2pt}{6pt plus 2pt minus 2pt}
\titlespacing\subsection{0pt}{6pt plus 2pt minus 2pt}{6pt plus 2pt minus 2pt}
\titlespacing\subsubsection{0pt}{2pt plus 0pt minus 0pt}{2pt plus 0pt minus 0pt}
\titleformat{\section}{\large\bfseries\sffamily}{\thesection}{1em}{}
\titleformat{\subsection}{\normalsize\bfseries\sffamily}{\thesubsection}{1em}{}
\titleformat{\subsubsection}{\small\sffamily}{\thesubsubsection}{1em}{}

\usepackage{amsmath,amsfonts,bm, amsthm}
\usepackage{algorithm}
\usepackage[noend]{algorithmic}
\usepackage{nicefrac}        
\usepackage{chngcntr}

\newcommand{\colref}[2]{\hyperref[#2]{#1~\ref*{#2}}}
\newcommand{\coloredref}[2]{\hyperref[#2]{#1~\ref*{#2}}}
\newcommand{\coloredsubref}[3]{\hyperref[#2]{#1~\ref*{#2}{#3}}}

\newcommand{\Figref}[1]{\colref{Figure}{#1}}

\newcommand{\Secref}[1]{\colref{Section}{#1}}

\def\eqref#1{\colref{Equation}{#1}}
\def\Eqref#1{\colref{Equation}{#1}}
\newcommand{\Eqnref}[1]{\colref{Equation}{#1}}

\newcommand{\Tabref}[1]{\colref{Table}{#1}}

\def\1{\bm{1}}

\DeclareMathAlphabet{\mathsfit}{\encodingdefault}{\sfdefault}{m}{sl}
\SetMathAlphabet{\mathsfit}{bold}{\encodingdefault}{\sfdefault}{bx}{n}

\DeclareMathOperator*{\argmin}{arg\,min}

\theoremstyle{plain}

\theoremstyle{remark}

\theoremstyle{definition}

\theoremstyle{plain}

\theoremstyle{plain}

\theoremstyle{definition}

\providecommand{\corollaryname}{Corollary}
\providecommand{\lemmaname}{Lemma}
\providecommand{\problemname}{Problem}
\providecommand{\remarkname}{Remark}
\providecommand{\theoremname}{Theorem}
\newtheorem{theorem}{Theorem}

\newcommand{\mvec}[1]{{\underline{#1}}}

\newcommand{\bunderline}[2][4]{\underline{#2\mkern-#1mu}\mkern#1mu}

\newcommand{\grad}{\bunderline{{\nabla}}}
\newcommand{\setbuilder}[1]{\left\{ #1 \right\}}

\newcommand{\inner}[2]{\left( #1, #2 \right)}

\newcommand{\ibn}{\textsc{IBN}}
\newcommand{\ibnfull}{Immersed Boundary Network}

\newcommand{\nel}{n_{el}}
\newcommand{\mesh}{\mathcal{K}^h}

\newcommand{\norm}[1]{\left|\left|{#1}\right|\right|}

\counterwithin*{theorem}{section}

\newcommand{\logLogSlopeTriangle}[5]
{
	
	\pgfplotsextra
	{
		\pgfkeysgetvalue{/pgfplots/xmin}{\xmin}
		\pgfkeysgetvalue{/pgfplots/xmax}{\xmax}
		\pgfkeysgetvalue{/pgfplots/ymin}{\ymin}
		\pgfkeysgetvalue{/pgfplots/ymax}{\ymax}
		
		\pgfmathsetmacro{\xArel}{#1}
		\pgfmathsetmacro{\yArel}{#3}
		\pgfmathsetmacro{\xBrel}{#1-#2}
		\pgfmathsetmacro{\yBrel}{\yArel}
		\pgfmathsetmacro{\xCrel}{\xArel}
		
		\pgfmathsetmacro{\lnxB}{\xmin*(1-(#1-#2))+\xmax*(#1-#2)} 
		\pgfmathsetmacro{\lnxA}{\xmin*(1-#1)+\xmax*#1} 
		\pgfmathsetmacro{\lnyA}{\ymin*(1-#3)+\ymax*#3} 
		\pgfmathsetmacro{\lnyC}{\lnyA+#4*(\lnxA-\lnxB)}
		\pgfmathsetmacro{\yCrel}{\lnyC-\ymin)/(\ymax-\ymin)} 
		
		\coordinate (A) at (rel axis cs:\xArel,\yArel);
		\coordinate (B) at (rel axis cs:\xBrel,\yBrel);
		\coordinate (C) at (rel axis cs:\xCrel,\yCrel);
		
		\draw[#5]   (A)-- node[pos=0.5,anchor=north] {1}
		(B)-- 
		(C)-- node[pos=0.5,anchor=west] {#4}
		cycle;
	}
}
\usepackage{hyperref}       
\hypersetup{
    colorlinks=true,
    linkcolor={blue},
    citecolor={blue},
    urlcolor={blue},
    breaklinks=true,
	plainpages=true
}
\theoremstyle{plain}

\newtheorem{lemma}[theorem]{Lemma}
\newtheorem{corollary}[theorem]{Corollary}
\theoremstyle{definition}
\newtheorem{assumption}[theorem]{Assumption}
\theoremstyle{remark}
\newtheorem{remark}[theorem]{Remark}

\begin{document}

\begin{center}
{\usefont{OT1}{phv}{b}{n}\selectfont\Large{Neural PDE Solvers for Irregular Domains}}

{\usefont{OT1}{phv}{}{}\selectfont\normalsize
{Biswajit Khara$^{1\dagger}$, Ethan Herron$^{1\dagger}$, Zhanhong Jiang$^4$, Aditya Balu$^2$, Chih-Hsuan Yang$^1$,\\ Kumar Saurabh$^1$, Anushrut Jignasu$^1$, Soumik Sarkar$^{1,2}$, Chinmay Hegde$^3$,\\ Adarsh Krishnamurthy$^{1,2}$, Baskar Ganapathysubramanian$^{1,2}$*}}

{\usefont{OT1}{phv}{}{}\selectfont\normalsize
{$^1$ Department of Mechanical Engineering, Iowa State University, Iowa, USA 50011\\
{$^2$ Translational AI Center, Iowa State University, Iowa, USA 50011\\
$^3$ Computer Science Department, New York University, New York, USA 10012\\
$^4$ Johnson Control Inc.\\
* Corresponding author: \href{mailto:barkarg@iastate.edu}{baskarg@iastate.edu}\\
$^\dagger$ Authors contributed equally
}}}
\end{center}

\begin{abstract}

Neural network-based approaches for solving partial differential equations (PDEs) have recently received special attention. However, the large majority of neural PDE solvers only apply to rectilinear domains, and do not systematically address the imposition of Dirichlet/Neumann boundary conditions over irregular domain boundaries. In this paper, we present a framework to neurally solve partial differential equations over domains with irregularly shaped (non-rectilinear) geometric boundaries. Our network takes in the shape of the domain as an input (represented using an unstructured point cloud, or any other parametric representation such as Non-Uniform Rational B-Splines) and is able to generalize to novel (unseen) irregular domains; the key technical ingredient to realizing this model is a novel approach for identifying the interior and exterior of the computational grid in a differentiable manner. We also perform a careful error analysis which reveals theoretical insights into several sources of error incurred in the model-building process. Finally, we showcase a wide variety of applications, along with favorable comparisons with ground truth solutions. 

\end{abstract}


\section{Introduction}
\textbf{Motivation} Most physical phenomena are modeled using a set of governing partial differential equations (PDEs). Numerical methods---finite difference (FDM), finite element (FEM), and spectral methods---for solving PDEs discretize the physical domain (into cells, elements, etc.) and \textit{approximate} the solution over this discretized domain using select families of basis functions~\citep{hughes2012finite,leveque2007finite,trefethen2000spectral}. A significant part of PDE solver technology involves solving PDEs on complex, irregular domains (for instance, flow across aerofoils in aeronautics or patient-specific organ geometries in medical diagnostics). This is a major challenge, as articulated in the NASA CFD 2030~\citep{slotnick2014cfd} vision (``Mesh generation and adaptivity continue to be significant bottlenecks...''). In particular, capturing the complex geometry, as well as rigorously accounting for the non-trivial boundary conditions on these complex geometries, has driven careers in mesh generation and PDE solver technology. 

Motivated by these challenges, this paper addresses the following problem, ``Can we design a neural PDE solver that can produce field solutions across arbitrary geometries?''. We do so by utilizing an analytical approach developed in the computational mechanics community---the immersed boundary method (IBM)~\citep{peskin2002immersed,saurabh2021industrial}. In such approaches, the irregular geometry is `immersed' in a regular grid, thus allowing standard meshes to model irregular geometries. We extend this powerful approach to neural PDE solvers, which enables the creation of PDE solvers that can produce field solutions for a distribution of irregular geometries. We show that such an approach also allows the natural incorporation of different boundary conditions over complex geometries; as a side benefit, our approach also allows us to compute \textit{a priori} error estimates using a combination of techniques from neural net generalization and finite element analysis.


\textbf{Neural PDE Solvers}
Since neural networks are powerful nonlinear function approximators, there has been a growing interest in using neural networks to solve PDEs \citep{raissi2019physics,kharazmi2019variational, sirignano2018dgm,yang2018physics, pang2019fpinns, karumuri2020simulator, han2018solving, michoski2019solving, samaniego2020energy, ramabathiran2021spinn, lu2019deeponet,botelho2020deep,balu2021distributed,wandel2021spline}. Unlike numerical methods, many of these methods do not require a mesh. But a common challenge most neural PDE solvers face is the efficient imposition of boundary conditions, especially on non-cartesian boundaries~\citep{lagari2020systematic}. 
Furthermore, among collocation-based neural solvers, the satisfaction of certain regularity conditions becomes non-trivial to impose~\citep{sukumar2022exact}. 



\textbf{Immersed Approach}
Classical numerical methods such as finite difference (FDM) or finite element methods (FEM) generally employ a grid or mesh to discretize the domain geometry and function space. 
Solving PDEs defined on complex geometries requires a mesh to be prepared before the analysis. This step, commonly known as the ``mesh generation'' step, is non-trivial and often expensive. 
Immersed methods~\citep{peskin2002immersed,mittal2005immersed} are one way to overcome this challenge. The computational grid is simplified in immersed methods by considering a rectilinear axis-aligned grid that encloses the irregular geometry (within which we seek a PDE solution). The irregularly-shaped geometry is then ``immersed'' in this background mesh (see Fig.~\ref{fig:domain-mesh}). Thus, a part of the background mesh forms the actual computational mesh; the rest of the background mesh is considered exterior, and thus not used in the computation of the PDE solution.

In this work, we combine the generalization capability of deep neural networks to solve PDEs defined on irregularly-shaped domains using ideas from immersed finite element methods~\citep{mittal2005immersed,zhang2004immersed,xu2016tetrahedral}. In recent years, immersed methods have been favored for massive parallel implementations since all computations are performed on a regular grid~\citep{bangerth2007deal,saurabh2021industrial,griffith2007adaptive,egan2021direct}. This fact translates to tensor-based operations in convolutional neural networks, which otherwise are unsuitable for dealing with complex geometrical contours. A key ingredient is the careful design of the loss function, along with a mechanism to determine the interior/exterior (in-out) of the computational domain. Our main contributions are as follows:


\begin{enumerate}[leftmargin=*, topsep=0pt, noitemsep, parsep=0pt]
    \item \textbf{[Framework]:} We present a parametric PDE-based loss function that learns robust and watertight boundary conditions imposed by complex geometries. Using this loss function, we train a deep neural network---Immersed Boundary Network (IBN)---that uses the geometric information as input to predict a field solution that satisfies the governing PDE over the arbitrary domain. We show that a single trained IBN can produce solutions to a PDE across a distribution of arbitrary shapes.
    \item \textbf{[Error Analysis]:} We provide analysis of convergence and generalization error bounds of the proposed PDE-based loss function with the immersed approach.
    \item \textbf{[Applications and Broader Impact]:} Finally, we use this parametric PDE-based loss function to learn solutions of PDEs over irregular geometries. We illustrate the approach on two PDEs---Poisson and Navier-Stokes---and a host of irregular geometries in 2D and 3D. IBN opens up fast design exploration and topology optimization for various societally critical applications such as shape design for energy harvesters, and aerodynamic design of vehicles. 
\end{enumerate}

\section{Mathematical Preliminaries}
\label{Sec:Formulation}

Consider domain $\Omega_B \subset \mathbb{R}^d $ ($ d\in \{2,3\} $) with a rectilinear boundary $ \Gamma_B $; and another domain $ \Omega_o \subset \Omega_B$ with an irregular boundary $ \Gamma_o $. Without loss of generality, the $ d $-dimensional unit interval $ [0,1]^d $ can be considered an example of $ \Omega_B $. Define $ \Omega = \Omega_B / \Omega_o $ (see \Figref{fig:domain-notations} for an illustration in 2D Euclidean space). On this non-rectilinear domain $ \Omega $, we now consider an abstract PDE (Eq 1a) with general boundary conditions (Eq 1b) given by:
\begin{subequations}\label{eq:pde-abstract}
	\begin{align}
	\mathcal{N}(u) &= f(\mvec{x}) \ \ \text{in} \ \ \Omega, \\
	\mathcal{A}(u) + \mathcal{D}(\nabla u) &= g \ \ \text{on} \ \ \Gamma_o \subset \partial\Omega.
	\end{align}
\end{subequations}
Here, $ u $ is an unknown function such that $ u: \Omega \mapsto \mathbb{R}$, $ \mathcal{N} $ is a differential operator (possibly nonlinear) and $ f(\mvec{x}) $ is a known function of the domain variable $ \mvec{x}\in \Omega $. The boundary conditions are defined by  operators $\mathcal{A}$ and $\mathcal{D}$. 
In this paper, we look at two concrete examples of \Eqnref{eq:pde-abstract}:  Poisson and Navier-Stokes equations.

\textbf{Poisson Equation}:
Poisson equation is frequently used to model steady-state mass/heat diffusion, electrostatics, surface reconstruction, etc. and is given by:
\begin{subequations}\label{eq:poisson-intro}
	\begin{align}
	-\Delta u &= f \ \ \text{in} \ \ \Omega, \\
	\alpha u + \beta \nabla u \cdot \mvec{n} &= g \ \ \text{on}\ \ \Gamma,
	\end{align} 	 
\end{subequations}
where $ u: \Omega\rightarrow \mathbb{R} $ is a scalar function that, depending on the underlying physics, represents the mass density, temperature, electric potential, or the surface indicator function, respectively. A spectrum of boundary conditions (Dirichlet, Neumann, Robin) are produced by varying the scalars $\alpha$ and $\beta$. 

\textbf{Navier-Stokes Equations}:
The Navier-Stokes equations are widely used to model fluid flow. The steady incompressible Navier-Stokes equations are given by:
\begin{subequations}\label{eq:ns-intro}
	\begin{align}
	\mvec{u}\cdot\grad\mvec{u} - \nu\Delta\mvec{u} + \grad p &= \mvec{f} \ \ \text{in} \ \ \Omega \\
	\mvec{u} &= \mvec{g} \ \ \text{on}\ \ \Gamma
	\end{align}
\end{subequations}
where $ \mvec{u}: \Omega\rightarrow\mathbb{R}^d $ is a vector valued function that represents the velocity field, and $ p: \Omega\rightarrow \mathbb{R} $ is a scalar representing the pressure in the fluid. The coefficient $ \nu $ is the viscosity of the fluid.

\begin{figure}[t!]
	\centering
	\begin{subfigure}[b]{0.58\linewidth}
	\includegraphics[width=0.99\linewidth,trim={3.3in 2.25in 3.3in 2.25in},clip]{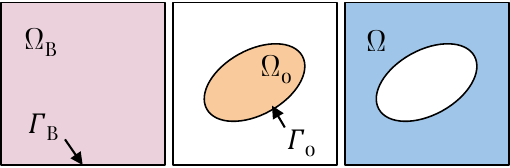}
	\caption{Computational Domain}
	\label{fig:domain-notations}
	\end{subfigure}
	\begin{subfigure}[b]{0.41\linewidth}
	\includegraphics[trim={1.2in 1.05in 1.2in 1.05in},clip,width=0.99\linewidth]{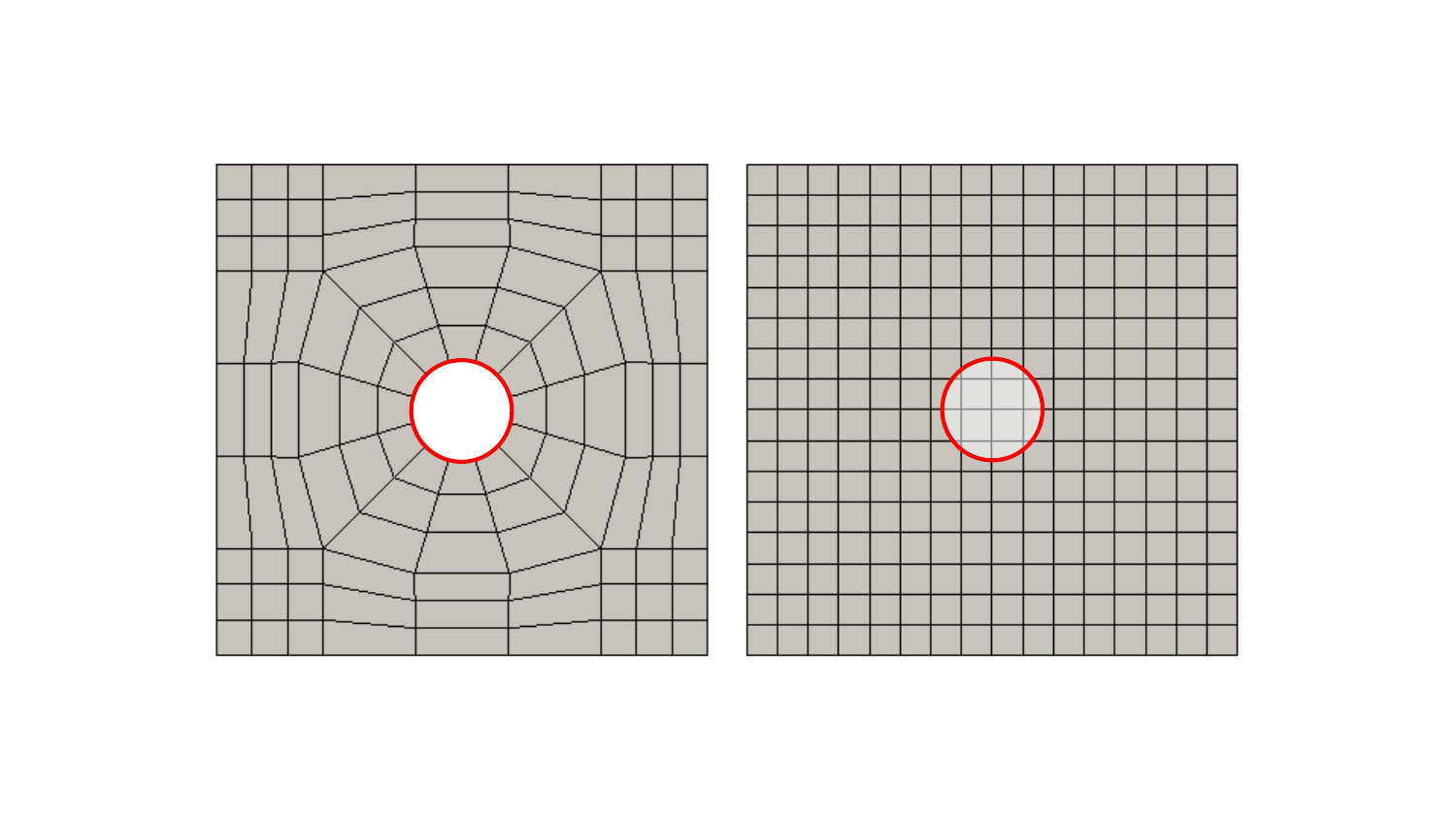}
	\caption{Mesh Types}
	\label{fig:mesh-types}
	\end{subfigure}
	\caption{(a) Schematic of a typical domain for the background mesh $ \Omega_B $ (with boundary $ \Gamma_B $), the embedded/immersed object $ \Omega_o $ (with boundary $ \Gamma_o $) and the computational domain $ \Omega $. (b) A ``body-fitted" mesh that discretely conforms to the object and an object embedded/immersed in a background mesh. This paper focuses on immersed type meshes with different object geometries.}
	\label{fig:domain-mesh}
\end{figure}


\textbf{Immersed Finite Element Method}: In this paper, we formulate our neural PDE solver based on an immersed FEM discretization. In terms of the notations developed above, we discretize $ \Omega_B $ using an axis-aligned grid (\Figref{fig:mesh-types} (right)) and we call this a ``background mesh''. To incorporate the effect of $ \Omega_o $, the boundary curve $ \Gamma_o $ is considered ``immersed'' in this background mesh, and the boundary conditions on $ \Gamma_o $ are applied approximately (in a weak manner) on the grid points that fall near this curve. The details of how the boundary conditions are applied are discussed in \Secref{sec:methodology}.

%

\section{Our Approach: \ibnfull{} (\ibn{})} \label{sec:methodology}


We will learn an \emph{immersed boundary network} (IBN) that produces a (field) solution to a given PDE and also (rigorously) adheres to the boundary conditions imposed by a family of complex geometries. 

Our IBN model is written as a parameterized function $G(\{P\};\theta)$, where $\theta$ is a set of tunable network weights, and $\{P\} (\subset \Gamma_o)$ represents the geometry (it can either be in the form of a point cloud or a set of NURBS control points representing the boundary of the complex geometry $ \Gamma_o $). Given such an input $\{P\}$, a trained network $ G $ should be able to predict both the interior/exterior of the computational domain and the solution to the PDE. Thus, the neural mapping is more precisely defined as  $G(\{P\};\theta): \mathbb{R}^{n\times d} \times \mathbb{R}^{|\theta|} \rightarrow \mathbb{R}^{N \times (1+ n_{dof})}$, where $n$ is the cardinality of $\{P\}$, $d$ is the spatial dimension, $N$ is the number of points used to discretize the mesh and $n_{dof}$ is the number of unknowns in the PDE (e.g., $n_{dof}=1$ for Poisson's equation, and $n_{dof}=(d+1)$ for Navier-Stokes equations). Suppose $\mathbf{U}$ is the discrete approximation of the solution field $u$ and $\boldsymbol{\chi}$ a discrete representation of the \emph{occupancy function} that indicates which points on the background mesh are interior/exterior to the immersed object. Then, given $ \{P\} $, we have $[\boldsymbol{\chi}, \mathbf{U}] = G(\{P\};\theta)$, where $ \boldsymbol{\chi} \in \mathbb{R}^{N \times 1}$ and $\mathbf{U} \in \mathbb{R}^{N \times n_{dof}}$, and the network architecture used to encode the geometry information can be either an MLP (for NURBS curves/surfaces) or Graph Convolution layers (for point clouds etc.). Further, for predicting fields, we use 2D/3D transpose convolution and upsampling layers (more details provided in Supplement).

\begin{figure}[t!]
	\centering
	\begin{subfigure}[b]{0.59\linewidth}
	\includegraphics[trim={3.1in 2.2in 3.1in 2.2in},clip,width=0.99\linewidth]{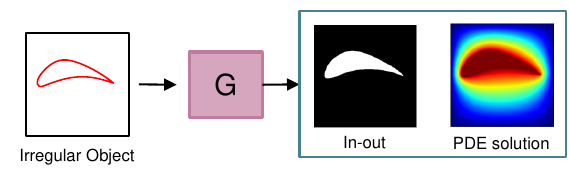}
	\caption{Input and output of the neural network}
	\label{fig:network-input-output-abstract}
	\end{subfigure}
	\begin{subfigure}[b]{0.4\linewidth}
	\includegraphics[width=0.99\linewidth,trim={2.0in 1.5in 2.0in 1.5in},clip]{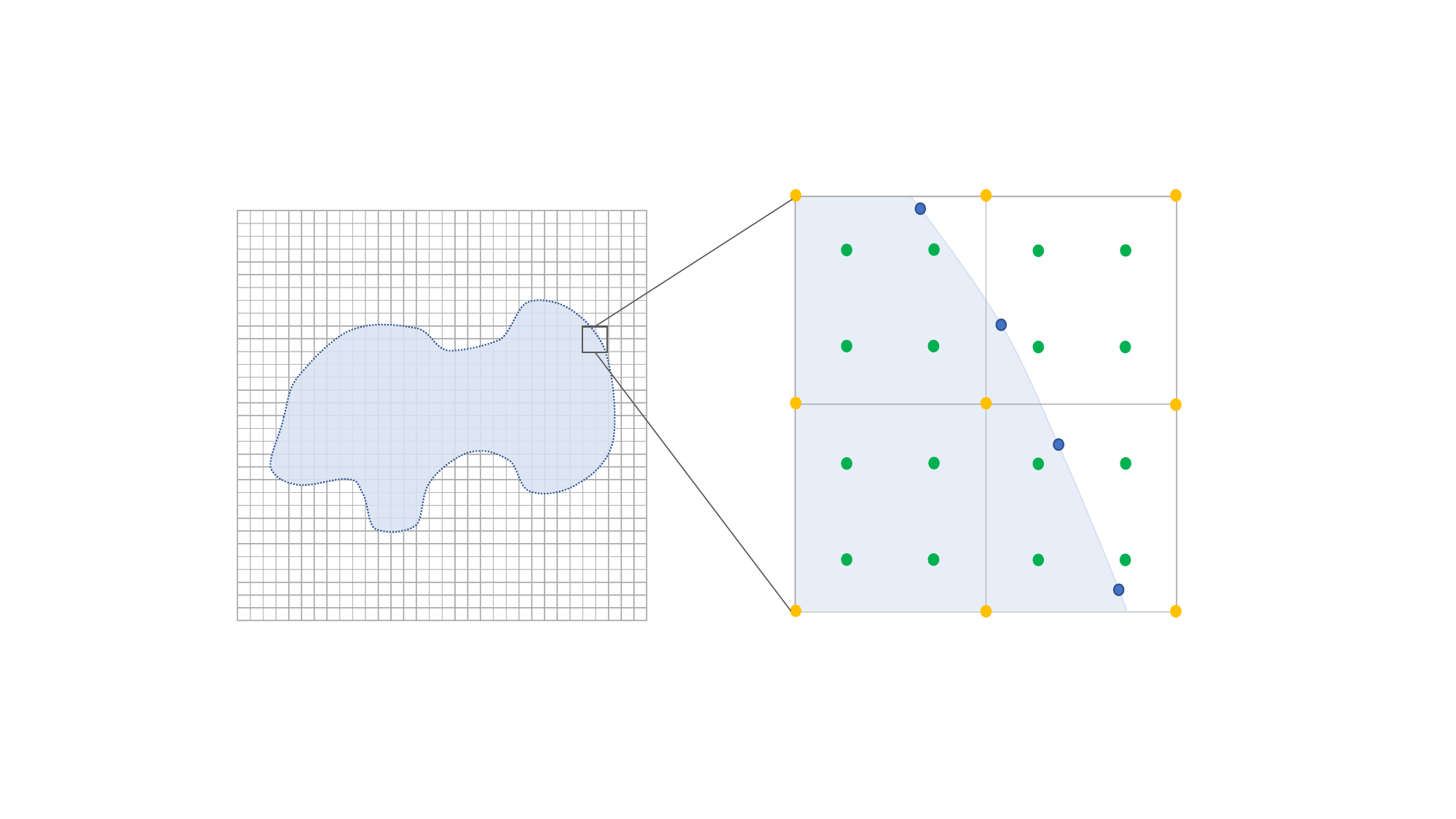}
	\caption{Immersed method}
    \label{fig:u_pts_grid}
	\end{subfigure}
	\caption{(a) The input to the network $ G $ is a representation of the object, possibly a point cloud or a NURBS  curve/surface. A trained network can then predict the in-out information of the object and the field solution. (b) A visual representation of the irregular object intersecting the element. Each yellow dot denotes the pixels; green dots denote the Gauss points used to interpolate the value of the finite element; blue dots denote points from the object boundary. The shaded region is the area inside the boundary.}
	\label{fig:immersed}
\end{figure}

\subsection{Numerical Framework}

We rely on a FEM discretization ($\mathbf{U}$, defined above) of the field solution  as well as its spatial derivatives. Let $ \mesh $ be a discretization of $ \Omega_B $ into $ \nel $ finite elements $ K_i $ such that $ \cup_{\nel}K_i = \Omega_B $ (Note, $\nel$ finite elements produce $N$ nodes, each of which have $n_{dof}$ predicted field values, resulting in the cardinality of the prediction of $G$). Following standard FEM analysis, we define a function space $V^h = \setbuilder{v^h \in V: v^h|_{K}\in Poly_m(K), \ K\in \mesh}$, where $ Poly_m(K) $ denotes the set of polynomial functions of degree $ m $ defined on $ K $. Since we are dealing with an immersed finite element method, $ \mesh $ is composed of a rectilinear axis-aligned grid with $ N $ nodes. Now, suppose  $ \{ \mathcal{B}_i(\mvec{x}) \}_{i = 1}^{N} $ is a suitable basis that span $ V^h $. Any function $ u^h \in V^h $ can be written as
\begin{align}\label{eq:def:fem-function-approximation-0}
u^h(\mvec{x}) = \sum_{i = 1}^{N} \mathcal{B}_i(\mvec{x})U_i,
\end{align}
where $ U_i $ are the solution values at the nodal points in the mesh $ \mesh $, with $\mathbf{U} = \{U_1, ..., U_i, ..., U_N\}$. In the sequel, when we refer to the discrete approximation of the solution, we mean either $ u^h $ (the functional representation) or the discrete set of nodal values $\{U_i\}$ (the vector representation). The two will be used interchangeably, assuming that the underlying basis functions $ \mathcal{B}(\mvec{x}) $ are known. 


\subsection{Training Loss Function}
A key novelty of our method lies in formulating an appropriate PDE loss function for irregular geometries and their associated numerics. At a high level, our loss consists of two parts: one addressing the given PDE and the other addressing the boundary conditions imposed by (complex, irregular) geometries. Thus we can break down the overall loss function into two main components: (i) predicting the field $u^h$ (PDE Loss), and (ii) predicting the occupancy function $\chi$ and exact imposition of boundary conditions (Geometry Loss).

\subsubsection{PDE Loss} The PDE component of the loss function is standard and follows directly from the Galerkin formulation of the given PDE. For the abstract PDE in \Eqnref{eq:pde-abstract}, the corresponding Galerkin formulation (with boundary conditions weakly satisfied) is to find $ u^h \in V^h$ such that
\begin{align}\label{eq:galerkin-weak-form-abstract}
	\inner{v^h}{[\mathcal{N}(u^h)-f]}_{\Omega} + \lambda\inner{v^h}{[\alpha u^h + \beta \nabla u^h \cdot n-g^h]}_{\Gamma} = 0, \ \ \forall v^h\in V^h.
\end{align}

In general, if a unique $ u^h $ exists, then for any other function $ w^h\in V^h $ ($ w^h\ne u^h $) will make the right hand side of \Eqnref{eq:galerkin-weak-form-abstract} nonzero. We call this the residual of \Eqnref{eq:galerkin-weak-form-abstract}, i.e.,
\begin{align}\label{eq:galerkin-residual-abstract}
	R = \inner{v^h}{[\mathcal{N}(w^h)-f]}_{\Omega} + \lambda\inner{v^h}{[\alpha w^h + \beta \nabla w^h \cdot n-g^h]}_{\Gamma},
\end{align} 
where, $\ w^h\in V^h $. Minimizing this residual should provide us with a unique solution. Unfortunately, we cannot perform integrations over $ \Omega $ in a straightforward manner, but we \emph{can} do so on $ \Omega_B $. On the other hand, $ \Omega_B $ also contains $ \Omega_o $ but we do not want to perform integrations on $ \Omega_o $. We resolve this via the occupancy function $ \chi^h $, which helps us perform integrations on $ \Omega_B $ while at the same time enforcing no contribution from $ \Omega_o $ to $ R $.

\subsubsection{Geometry Loss}\label{sec:geometry-loss} 

This corresponds to imposition of boundary conditions on the surface defined by points $\mvec{p}_i \in \{P\} \subset \Gamma_o$  and correct prediction of the occupancy function $ \chi $. For given surface points $\mvec{p}_i$, the discrete solution $u^h(\mvec{x})$ on these boundary points is $u(\mvec{p}_i) = \sum_{j = 1}^{N} \mathcal{B}_j(\mvec{p}_i) U_j~~\forall \mvec{p}_i \in \{P\}$. This interpolation (separate from the Gauss points, see \Figref{fig:immersed}(b)) helps us exactly impose the boundary conditions. (More details on the interpolation are explained in Supplement).

Predicting the occupancy function $ \chi $ can be obtained using two approaches : (i) computing winding numbers, or (ii) solving the Eikonal equation as another unknown.

\textbf{Generalized Winding Number:} This is the general form of the winding number which counts the number of loops of the boundary around any point $\mvec{x}$. The generalized winding number can be calculated as the surface integral of the differential solid angle ($\int_{\Gamma_o} dS(\mvec{x})$)~\citep{barill2018fast}. Given a boundary representation of an object, we can use the generalized winding number to determine in-out information since any point inside the object boundary will have a winding number $1$. The in-out occupancy $\chi_w$ is calculated as:
\begin{equation}
\chi_{w}(\mvec{x}) = \sum_{i=1} ^{n} a_i \frac{(\mvec{p_i} - \mvec{x}) \cdot \hat{\mvec{n}}_i}{4\pi || \mvec{p}_i - \mvec{x} ||^3}
\label{winding numbers}
\end{equation}
where $a_i$ is the Voronoi area of a given point and $\mvec{n_i}$ is the normal for a given point $\mvec{p_i}$~\citep{barill2018fast}. Using $\chi_w$, we can obtain all nodes inside the geometry ($u^h \in \Omega_o$) by using $\chi_{w}$ as a masking function. Since this is deterministically obtained, the network $G$ only predicts $N \times n_{dof}$ values. However, there are some limitations of $\chi_{w}$ not being robust to noise. Hence, an alternate approach using signed distance fields (SDF) has recently become popular.

\textbf{Eikonal Equation:} SDF ($\phi$) denotes the distance to the closest point on the surface $\Gamma_o$. The sign of the distance fields can be used as occupancy information ($\chi_e = sign(\phi)$). The signed distance field is obtained by solving the viscosity stabilized version of the Eikonal equation ($(1+\tau)\|\grad \phi\| - \tau (\nabla \phi) = 1$, $\tau \in [0,0.5]$)~\citep{phasetransitionslipman}. For this, the network $G$, will predict one additional field, i.e., $N \times (1 + n_{dof})$ values, unlike the Winding number-based approach. In this work, we demonstrate the usage of both $\chi_e$ and $\chi_w$ for obtaining the mask and applying the boundary conditions.

Combining the PDE Loss and the Geometric Loss, the complete loss function is:
\begin{align}\label{eq:combined-loss-winding-num}
	\mathcal{J}(w^h) = \underbrace{\norm{R(w^h)}_2}_{\text{PDE loss}}  + \underbrace{\lambda_1 \sum_{i = 1}^{{n}} |\alpha w^h(\mvec{x}_i) + \beta \nabla w^h(\mvec{x}_i) \cdot n - g_i|^2}_{\text{boundary condition on object}} + \underbrace{ \lambda_2 \sum_{j\in \{j|\chi >0\}} |w^h - g_{\text{in}}|^2}_{\text{exterior points}},
\end{align}
where $ g_{\text{i}}$ and $ g_{\text{in}} $ represent suitable functions defined in $\Gamma_o$ and $ \Omega_o $ respectively. We assume that $g_{\text{i}}$ conforms to the boundary $\Gamma_o$. i.e., the point cloud is not noisy, and the boundary conditions are applied accurately. 
Thus we can then define the minimization problem statement as
\begin{align}
u^h = \argmin_{w^h\in V^h} \mathcal{J}(w^h).
\end{align}
which can be approximately solved using multiple gradient descent steps. To summarize, the first term in \Eqnref{eq:combined-loss-winding-num} represents the loss function derived from the PDE, the second term attempts to ``weakly'' set the boundary points to their respective prescribed values, and finally, the third term attempts to set the object interior points to a fixed value conforming to the object boundary. Appropriate values for $ \lambda_1 $ and $ \lambda_2 $ are generally problem dependent, but for our experiments we take $ \lambda_1 = \lambda_2 = O(1) $. However, in case of high gradient in the solution values near boundaries, these coefficients scale $ \propto \frac{1}{h} $. This closely aligns with the results for these coefficients from the analysis of the IBM \citep{ern2021finite,schillinger2016non}.

\section{Error Analysis}\label{theory}
In this section, we summarize theoretical results for the convergence behavior and the generalization error for \ibn{}. Due to tight space constraints, all detailed proofs are relegated to the supplementary material. A summary of the theoretical results is presented in \Tabref{table:simDisimCoefNewDef}.

\begin{table}[t!]
\caption{Summary of different generalization bounds. Please refer to the Appendix for the detailed derivation of Corollary~\ref{coro_2} and Corollary~\ref{coro_3}.}
\centering
\begin{tabular}{l l l l l}
    \toprule
    \textbf{Statement} & \textbf{Setting} & \textbf{Property} &\textbf{Step Size}& \textbf{Rate} \\ \midrule
      Corollary~\ref{coro_1_main}   &      Batch               &    $\mu$-strong convexity                &Constant& $\mathcal{O}\bigg(N^2\rho_1^K+\frac{1}{(NK)^{\alpha}}\bigg)$ \\
      Corollary~\ref{coro_2}   &   Stochastic                  &   $\mu$-PL condition                & Constant& $\mathcal{O}\bigg(N^2\rho_2^K+N^2\sigma^2+\frac{1}{(NK)^{\alpha}}\bigg)$\\ 
      Corollary~\ref{coro_3}   &    Stochastic                 & $\mu$-PL condition              & Diminishing& $\mathcal{O}\bigg(\frac{N^2}{K}+\frac{1}{(NK)^{\alpha}}\bigg)$\\\bottomrule
\end{tabular}
\begin{tablenotes}
\item[1] $\mu$: the strongly convex constant defined in the later section; \item[2] $N$: the number of basis functions; \item[3] $\rho_1,\rho_2\in(0,1)$; \item[4] $K$: the number of epochs; \item[5] $\alpha$: the polynomial order of FEM basis functions; \item[6] PL: Polyak-Lojasiewicz; \item[7] $\sigma$: second moment constant.
\end{tablenotes}
\label{table:simDisimCoefNewDef}
\end{table}

For a given geometry $\Gamma_o$, we denote $u$ as the original optimum solution to the PDE and $u^h$ be the optimal solution at discretization level $h$; these are functions evaluated at any point $\mvec{x} \in \Omega$. 
From classical FEM analysis, we will bound the discretization error $\| u^h - u \|$ as a function of $h$.
Now, the field predicted by the network $G(\{P\}, \theta)$ is typically an inaccurate version of $u_h$, and we can express the error $e_G$ as follows:
\begin{equation}\label{dynamic_gen_err}
    \|e_G^k\|=\|u_{\theta_k}-u\| = \|u_{\theta_k}-u_{\theta^*}+u_{\theta^*}-u^h+u^h-u\|
\end{equation}
where, $u_{\theta_k}$ is the field solutions corresponding to the network parameters $\theta_k$ at the $k^{th}$ iteration of the optimization and $u_{\theta^*}$ is the theoretical optimum (i.e., a limit point) of $u_{\theta_k}$.

We analyze the quadratic form of $\|e^k_G\|$, i.e., $\|e^k_G\|^2$; in the stochastic setting, the relevant quantity is the second moment, i.e., $\mathbb{E}[\|e^k_G\|^2]$. By applying the fundamental inequality $\|a+b+c\|^2\leq3(\|a\|^2+\|b\|^2+\|c\|^2)$, we can obtain
\begin{equation}\label{errorterms}
    \|e^k_G\|^2\leq 3(\underbrace{\|u_{\theta_k}-u_{\theta^*}\|^2}_{\text{Optimization error}}+\underbrace{\|u_{\theta^*}-u^h\|^2}_{\text{Modeling error}}+\underbrace{\|u^h-u\|^2}_{\text{Discretization error}}).
\end{equation}

Thus the generalization error is bounded above by three different terms, which are respectively the \textit{optimization error}, \textit{modeling error}, and \textit{discretization error}. The discretization error is only dependent on the choice of the discretization and the modeling error is based on the choice of the network architecture and hence both the errors remains static through out the optimization.

To facilitate the understanding of how $\|e^k_G\|^2$ converges and to simplify the analysis, we will next combine the optimization error and modeling error, while following standard FEM analysis to bound the discretization error. We make the following assumptions.

\begin{assumption}\label{assumption_1}
There exists a constant $c>0$ such that $\|u_{\theta^*}-u^h\|\leq c \|u_{\theta_k}-u_{\theta^*}\|, \forall k\geq 0$.
\end{assumption}
Assumption~\ref{assumption_1} implies that the modeling error can be upper bounded by the optimization error up to a constant. This should be possible since we can always find some constant $c$ to ensure that the above assumption holds. 
Therefore, we now obtain
\begin{equation}\label{dynamic_gen_err_1}
    \|e^k_G\|^2\leq 3[(1+c^2)\|u_{\theta_k}-u_{\theta^*}\|^2+\|u^h-u\|^2]. 
\end{equation}

\begin{assumption}\label{assumption_2}
There exists a constant $L>0$ such that for all $\mvec{x}, \mvec{y} \in\mathbb{R}^d$, $\|U_i(\mvec{x})-U_i(\mvec{y})\|\leq L\|\mvec{x}-\mvec{y}\|$, for all $i\in\{1,2,...,N\}$.
\end{assumption}
Assumption~\ref{assumption_2} implies that the prediction function provided by \ibn{} is Lipschitz continuous~\citep{virmaux2018lipschitz}, which signifies the robustness of the predictions obtained from deep neural networks under perturbations of the network parameters. This is a standard assumption made during the analysis of neural network generalization. 

We first investigate the convergence rate when $\mathcal{J}$ is \textit{strongly convex and in a batch setting}. Notice that all terms in \Eqnref{eq:combined-loss-winding-num}, including the boundary loss (which is $L_2$ loss of the imposed boundary conditions) and the PDE residual loss (from variational arguments) are strongly convex.

\begin{lemma}\label{lemma_4}
Assume that the basis function $\mathcal{B}_i(\mvec{x})$ are chosen such that they are at least continuously differentiable locally over a mesh. Then the following relationship holds true:
\begin{equation}
    \|u^h-u\|^2\leq Ch^{2\alpha},
\end{equation}
where $C>0$ and $\alpha$ is the order of continuous derivative. Typically, $\alpha\geq 1$.
\end{lemma}

\begin{theorem}\label{theorem_1_main}
Suppose that $\mathcal{J}$ is $\mu$-strongly convex and $\beta$-smooth. Let Assumptions~\ref{assumption_1} and~\ref{assumption_2} hold. By applying the gradient descent algorithm with a constant step size $\eta_k=\eta=\frac{2}{\mu+\beta}$, the generalization error $\|e^K_G\|^2$ satisfies the following relationship after $K$ epochs,
\begin{equation}
    \|e^K_G\|^2\leq 3(1+c^2)L^2N^2\hat{\mathcal{B}}^2\bigg(1-\frac{2}{\kappa+1}\bigg)^{2K}\|\theta_0-\theta^*\|^2+3Ch^{2\alpha},
\end{equation}
where $c, N$ are defined in Assumptions~\ref{assumption_1} and~\ref{assumption_2}, $\mu$ is the strong convexity constant, $\beta$ is the smoothness constant, $L$ is the Lipschitz continuous constant, $\hat{\mathcal{B}}$ is the upper bound of the basis function $\mathcal{B}_i(\mvec{x})$, $\kappa=\frac{\beta}{\mu}$ is the condition number, and $\theta_0$ and $\theta^*$ are respectively the initialization and optimum of $\theta$.
\end{theorem}

Theorem~\ref{theorem_1_main} implies that the generalization error of \ibn{} is upper bounded by two terms, with the first term related to the optimization error and the second term the static error due to the discretization. Additionally, in an asymptotic manner, we have that
\begin{equation}
    \lim_{K\to\infty}\|e^K_G\|^2\leq 3Ch^{2\alpha},
\end{equation}
which suggests that the generalization error will ultimately be dominated by the discretization error determined by the resolution $h$ and the order of the FEM basis function $\alpha$ (here $\alpha = 1$), both of which rely on the definition of basis functions $\{\mathcal{B}_i(\mvec{x})\}_{i=1}^N$. 
If $h$ is chosen such that $h=\mathcal{O}(\frac{1}{\sqrt{KN}})$ (justified by~\citep{larson2013finite}), the following corollary can be obtained:
\begin{corollary}\label{coro_1_main}
Suppose that $h=\mathcal{O}(\frac{1}{\sqrt{KN}})$. Given all conditions and parameters from Theorem~\ref{theorem_1_main}, the generalization error achieves an overall \textit{sublinear} convergence rate
\begin{equation}
    \|e^K_G\|^2\leq\mathcal{O}\bigg(N^2\rho_1^K+\frac{1}{(NK)^{\alpha}}\bigg).
\end{equation}
where $\rho_1$ is a constant smaller than 1.
\end{corollary}



In summary, our upper bound on the generalization error can be separated into two terms, both of which converge to zero as the number of training epochs $K$ become large --- provided the discretization is also chosen inversely with $K$. Qualitatively, the above bound provides a thumb rule to set the discretization of the mesh, $h$, based on available computation time (measured in the number of epochs $K$) and the desired error level $e_G$. From a practical perspective, the bound allows estimation of resource requirement for a desired discretization and error level.

\section{Experimental Results}\label{sec:results}
In this section, we provide results from our proposed immersed neural approach. We highlight single instance field solutions for Poisson's and Navier-Stokes equations and the field solutions to Poisson's equation for a family of randomly generated shapes. As previously mentioned, the IBN framework can map geometry representation to the field solution of a given PDE. While the geometry representation could be an unstructured point cloud or other structured representations such as NURBS control points, the PDE field solution is represented as a uniform grid. We use a single Nvidia Tesla P40 GPU in all the experiments shown below. Each network used the Adam optimizer with a learning rate of 3e-4. Further, as explained in \Secref{sec:geometry-loss}, we have two approaches for obtaining the occupancy information($\chi$). A comparison between both the methods is provided in the Supplement, and we only show generalized winding number-based results here.

\begin{figure}[b!]
	\centering
    \includegraphics[trim=0.2in 0.1in 0.1in 0.55in,clip,width=0.9\linewidth]{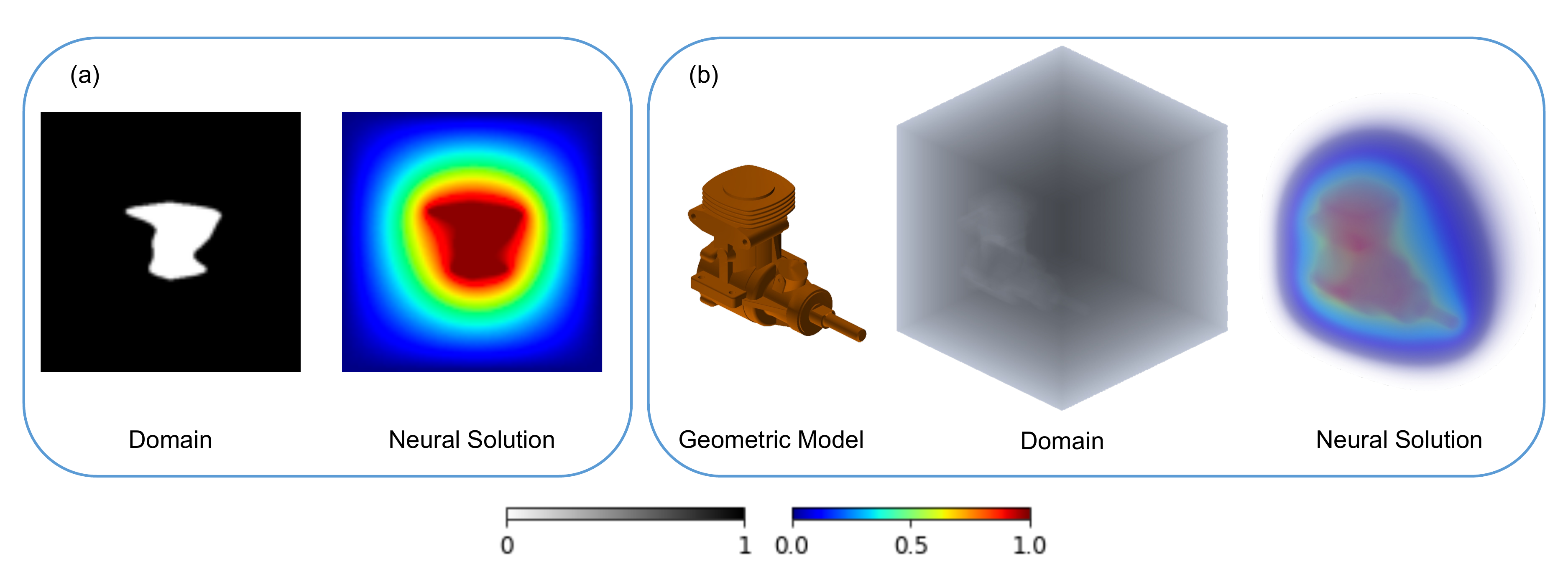}
	\caption{(a): Left: The in-out of the domain for the immersed solve. Right:~Solution to the Poisson's equation with an immersed object in 2D. (b) Left: The geometry model being immersed. Center: The in-out information. Right: Solution to Poisson's equation using the IBN framework.}
	\label{fig:imdiff-poisson-results}
\end{figure}

\subsection*{Solution to Poisson's Equation in 2D and 3D:}
For solving Poisson's Equation in 2D, we first generate a large geometry data (a family of NURBS curves sampled randomly). We ensure that the shapes are smooth and do not have any discontinuities and self-intersections. Once the shapes are generated, we obtain point clouds composed of 1,000 points, and all field solutions were represented using a $128\times128$ resolution grid (represented as an image). Our neural network to map the point cloud to the image domain was composed of 6 MLP layers followed by 8 layers of convolution and transpose convolution operations. Here, we choose simple MLP layers, but more sophisticated architectures such as DGCNNs could also be used. \Figref{fig:imdiff-poisson-results}(a) shows the results for one of the test geometries (the occupancy information $\chi$ and the solution to Poisson's equation $U$).

The same problem can be solved using 3D parametric datasets such as the ABC dataset, ShapeNet, etc. We use a subset of geometries from the ABC dataset to show results here. The process remains the same, except for the resolution of the grid being $196\times196\times196$ in 3D. \Figref{fig:imdiff-poisson-results}(b) shows one of the 3D geometries used (an Engine), the corresponding domain, and the volume rendering of the neural solution to the Poisson's equation.

\subsection*{Convergence Analysis:}
To check the accuracy of our method, we solve Poisson's equation \Eqnref{eq:poisson-intro} on a circular disk of radius $ R = 0.25 $ immersed in a unit square. The forcing function $ f = 1 $ on the disk, and the boundary condition is specified as $ u = 0 $ at the perimeter of the disk. Therefore, $ \Omega_o = \{r<R\}$ and $ \Gamma_o = \{r=R\} $. The exact solution to this problem is given by $ u(r) = \frac{1}{4}\left(R^2-r^2\right) $, where $ r = \sqrt{x^2+y^2} $ is the radial position of a point on the disk.

\begin{figure}[h!]
    \begin{subfigure}[b]{0.67\linewidth}
		\includegraphics[trim=0 0 0 0,clip,width=0.99\linewidth]{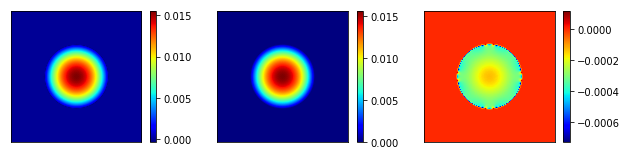}
		\caption{Contours of $ u^h $, $ u $ and $ u^h-u $}
		\label{fig:poisson-disk-contours}
	\end{subfigure}
	\hspace{0.02\linewidth}
    \begin{subfigure}[b]{0.295\linewidth}
		\centering
		\begin{tikzpicture}
		\begin{loglogaxis}[
		width=0.99\linewidth, 
		xlabel={\scriptsize{$h$}}, 
		ylabel={\scriptsize$\norm{e}_2$},
		legend style={at={(0.01,0.99)},anchor=north west,legend columns=1}, 
		x tick label style={rotate=0,anchor=north}, 
		xtick={0.001953125, 0.00390625, 0.0078125, 0.015625 , 0.03125  , 0.0625},
		xticklabels={\scriptsize $2^{-9}$,\scriptsize $2^{-8}$,\scriptsize $2^{-7}$, \scriptsize $2^{-6}$, \scriptsize $2^{-5}$,\scriptsize $2^{-4}$},
		ymin=5e-6,
		ymax=0.001,
		ytick={1e-7,1e-5,1e-4,0.001},
		]
		\addplot table[x expr={\thisrow{h}},y expr={\thisrow{err}},col sep=comma]{Figures/convergence-test/errors.txt};
		\logLogSlopeTriangle{0.8}{0.2}{0.3}{1}{black};
		\end{loglogaxis}
		\end{tikzpicture}
		\caption{Error convergence}
		\label{fig:poisson-disk-convergence}
	\end{subfigure}
	\caption{Error convergence with linear basis ($ \alpha = 1 $) for the Poisson's equation solved on a circular disk of radius $ R=0.25 $.}
	\label{fig:poisson-disk-results}
\end{figure}
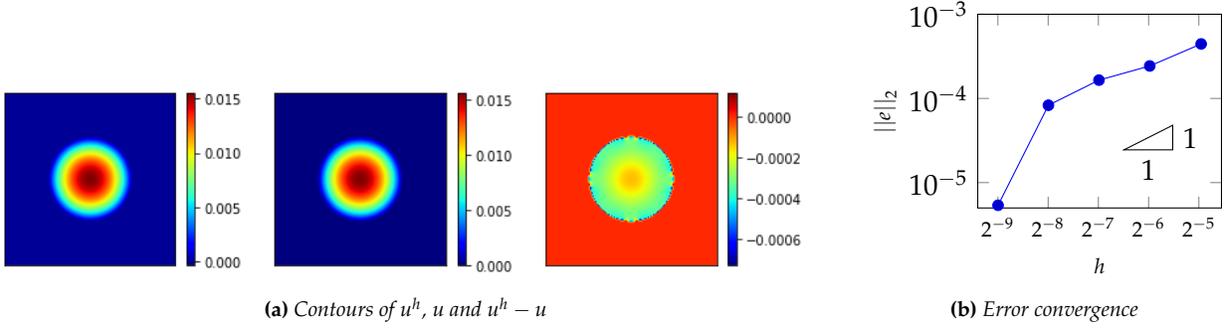

We discretize the unit square domain using an $ N\times N $ grid, and the object boundary $ \Gamma_o $ is represented using a point cloud. We optimize the loss function mentioned in \Secref{sec:methodology} to obtain the solution. The results are presented in \Figref{fig:poisson-disk-results}. \Figref{fig:poisson-disk-contours} shows the contours of the discrete solution $ u^h $, the exact solution $ u $ and the error $ e = u^h - u $, for $ N=128 $. \Figref{fig:poisson-disk-convergence} shows the $ L^2 $-norm of the error with mesh size $ h $ in a log-log plot. We see first-order convergence with increasing resolution, which matches established theory from IBM analysis~(\citep{ern2021finite} (Lemma 37.2) and \citep{schillinger2016non}).


\subsection*{Navier-Stokes Equation:}
We also validate our framework against a canonical flow past an airfoil using the steady Navier-Stokes equations (NSE) introduced in \Eqnref{eq:ns-intro}. The flow field output from \ibn{} shows the expected wake structure (at Reynolds number 40), is symmetric about the mid-plane, shows the stagnant pressure point on the upwind side of the immersed object, and satisfies the imposed no-slip condition. 

The boundary conditions for this problem are:
\begin{subequations}
	\begin{align}
		x = 0 &: \ \ u_x = 1 - \left(\frac{2y}{H}-1\right)^2, \ \ u_y = 0 \\
		y = 0 &: \ \ u_x = 0, u_y = 0 \\
		y = H &: \ \ u_x = 0, u_y = 0 \\
		(x,y) \in \Gamma_o &:\ \  u_x = 0, u_y = 0.
	\end{align}
\end{subequations}
\Figref{fig:imdiff-ns-fps-non-network} shows the $x$ velocity, $y$ velocity, and the pressure solution for the Navier Stokes equation using \ibn{} for an aerofoil.

\begin{figure}[t!]
	\centering
	\includegraphics[trim={1.0in 0.0in 1.0in 0.0in},clip,width=0.99\linewidth]{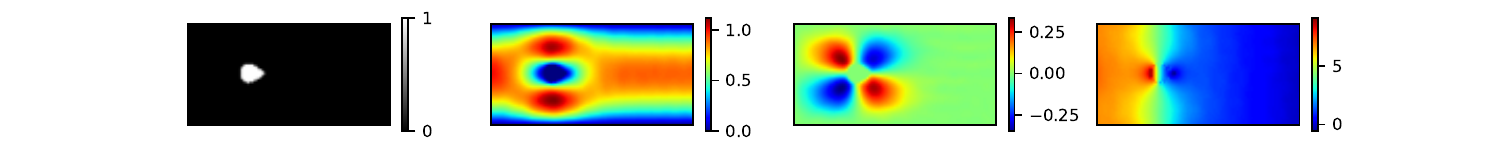}
	\caption{Immersed method with object mask applied to solve the Navier-Stokes equation for a steady flow past a NACA 0012 aerofoil. (left) domain/boundary mask  (middle left) $x$ velocity, (middle right) $y$ velocity, and (right) pressure. }
	\label{fig:imdiff-ns-fps-non-network}
\end{figure}

\subsection*{Parametric Steady State Heat Equation:}
We revisit Poisson's equation for a family of objects characterized by boundary $ \Gamma_o $. Specifically, we consider the object boundary to be taken from a set of parameterized curves. The parametrization here is a NURBS-based one. This is equivalent to having access to the point cloud representing the object. We show results for two distinct distributions of objects. The first example consists of irregularly shaped objects without obvious parametrization, while the second consists of a parametrizable family of one thousand NACA airfoils. \Figref{fig:imdiff-d1d2} shows some anecdotal example predictions of such a trained network. Each example consists of a \textit{single trained IBN}, which is then queried for a set of unseen object boundaries. \ibn{} can accurately predict the field solution under non-trivial geometries, as a comparison with a high-resolution FEM solver confirms (columns 3, 6, 9, 12).  


\begin{figure}[t!]
	\centering
	\begin{subfigure}[b!]{0.49\linewidth}
	\includegraphics[trim=63 40 50 40,clip,width=0.49\linewidth]{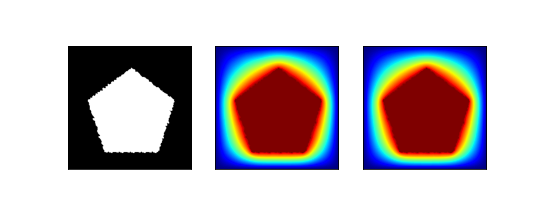}
	\includegraphics[trim=63 40 50 40,clip,width=0.49\linewidth]{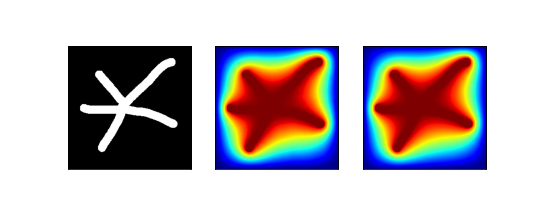}\\
	\includegraphics[trim=63 40 50 40,clip,width=0.49\linewidth]{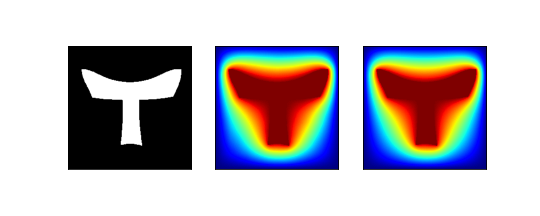}
	\includegraphics[trim=63 40 50 40,clip,width=0.49\linewidth]{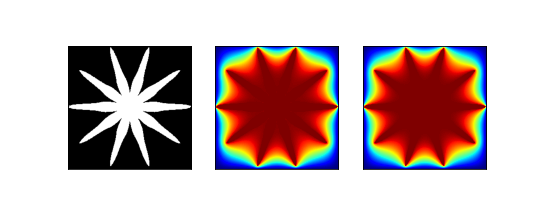}
	\caption{Dataset 1}
	\end{subfigure}
	\begin{subfigure}[b!]{0.49\linewidth}
	\includegraphics[trim=63 40 50 40,clip,width=0.49\linewidth]{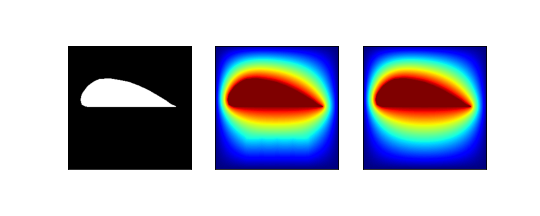}
	\includegraphics[trim=63 40 50 40,clip,width=0.49\linewidth]{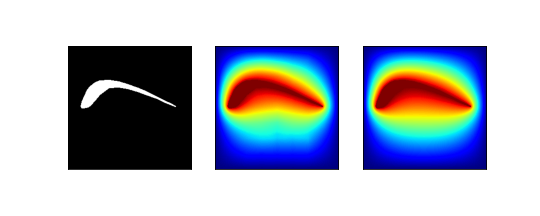}\\
	\includegraphics[trim=63 40 50 40,clip,width=0.49\linewidth]{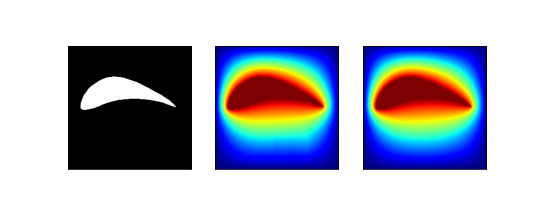}
	\includegraphics[trim=63 40 50 40,clip,width=0.49\linewidth]{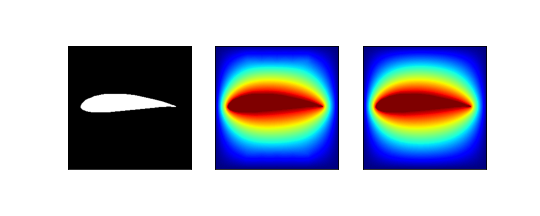}
	\caption{Dataset 2}
	\end{subfigure}
	\includegraphics[trim=0 0 0 0,clip,width=0.4\linewidth]{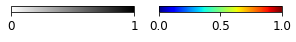}
	\caption{\ibn{} applied on two different datasets. Within each panel: the left shows the heat source in white, the middle shows the \ibn{}-predicted temperature distribution, and the right shows the fully converged numerical solution.}
	\label{fig:imdiff-d1d2}
\end{figure}


\subsection*{Comparisons:}
We compare \ibn{} to standard FE-based result for Poisson's equation in 2D and compare it with one of the state-of-the-art methods, PINNs\citep{raissi2019physics}. \Tabref{tab:comparisons} shows comparison of \ibn{} to numerical solution obtained using traditional approaches such as \citep{bangerth2007deal,saurabh2021industrial,griffith2007adaptive,egan2021direct}. Similarly, we show these comparisons for PINNs as well. Our approach is consistently more accurate than PINNs in predicting the field solutions. Also, the training time for PINNs is much higher than the training time for \ibn{}. Since PINNs can deal with only one geometry at a time, we show similar training behavior, but our approach can generalize to more than one geometry.

\begin{table}[h!]
\centering
\small
\setlength\extrarowheight{1pt}
\vspace{-0.1in}
\caption{Comparison of L2 norm of errors incurred by the  proposed IBN method and PINN, respectively when compared with  a baseline numerical solver. and timing for prediction of solution from PINNs and \ibn{} }\label{tab:comparisons}
\begin{tabular}{|l|c|c||c|c|}
\hline
Geometry  & $||u_{PINN} - u_{FEM}||$&$||u_{IBN} - u_{FEM}||$ & \begin{tabular}{c}
Training time  \\
PINN (s)
\end{tabular} & \begin{tabular}{c}
Training time  \\
\ibn{} (s)
\end{tabular} \\
\hline
\hline
2032C     & 0.1638 & 0.0158 & 3667.83 & 329.35\\
A18       & 0.1745 & 0.0156 & 3893.29 & 286.41\\
NACA 0012 & 0.1483 & 0.0158 & 3890.78 & 245.64\\
NACA 0010 & 0.1536 & 0.0157 & 2031.78 & 279.98\\
NACA 0021 & 0.1471 & 0.0160 & 3666.13 & 244.36\\
\hline
\end{tabular}
\end{table}

\subsection*{Limitations:}
One of the major limitations is that compared to other neural PDE solvers, our proposed approach is more memory intensive while being better and faster than them. Further, the two occupancy information calculation approaches have a trade-off. While the winding number-based approach is easier to calculate and faster, they are not robust to noise and require huge memory resources to perform the computation. Similarly, the Eikonal equation-based approach is more robust to noise in the geometry representation and has a smaller memory footprint, but suffers from a long time to compute the SDF accurately and causes initial stability issues.

\section{Conclusions}
We have developed a neural PDE solver that can handle irregularly shaped domains by building on well-established finite element and immersed boundary methods. Our neural PDE solver, coined IBN, demonstrates the ability to predict field solutions for irregular boundaries immersed in the target domain. We highlight two specific PDE cases, Poisson and Navier-Stokes, which show promising results. Alongside the empirical results, we have included theoretical results for the error bounds of the optimization process of our finite element-based loss function. IBN opens up fast design exploration and topology optimization for various societally critical applications such as room ventilation for reduced disease risk, shape design for energy harvesters, and aerodynamic design of vehicles.  

\vfill
\pagebreak
{
\small
\bibliographystyle{unsrtnat}
\bibliography{Bibs}
}

\vfill
\pagebreak


\appendix
\setcounter{page}{1}
\section*{Appendix}

\section{Theory for Generalization Error}\label{theoryA}
In this section, we show in detail the convergence behavior of the generalization error for \ibn{}. For a given surface of the geometry, $\Gamma_o$ and let $u^h$ be the discrete optimum solution to the PDE, evaluated at any point $\mvec{x} \in \Omega$. Note that, $u^h$ need not be the optimum solution to the PDE itself and we denote $u$ as the original optimum solution to the PDE. Hence, $\| u^h - u \|$ is a known error due to the discretization of $\Omega_B$. We will have more discussion about this later in this section.



Now, given the discrete optimum solution $u^h$ predicted by the network $G(\{P\}, \theta)$, we can express the generalization error $e_G$ as follows:
\begin{equation*}\tag{\ref{dynamic_gen_err}}
    \|e_G^k\|=\|u_{\theta_k}-u\| = \|u_{\theta_k}-u_{\theta^*}+u_{\theta^*}-u^h+u^h-u\|
\end{equation*}
where, $u_{\theta_k}$ is the field solutions corresponding to the network parameters $\theta_k$ at $k^{th}$ iteration of the optimization and $u_{\theta^*}$ is the theoretical optimum of $u_{\theta_k}$ that an algorithm can achieve ideally.

To follow the convention in optimization domain, we analyze the quadratic form of $\|e^k_G\|$, i.e., $\|e^k_G\|^2$, and in a stochastic setting, it is equivalent to the second moment, i.e., $\mathbb{E}[\|e^k_G\|^2]$. By applying a fundamental inequality $\|a+b+c\|^2\leq3(\|a\|^2+\|b\|^2+\|c\|^2)$, we can obtain
\begin{equation*}
    \tag{\ref{errorterms}}
    \|e^k_G\|^2\leq 3(\underbrace{\|u_{\theta_k}-u_{\theta^*}\|^2}_{\text{Optimization error}}+\underbrace{\|u_{\theta^*}-u^h\|^2}_{\text{Modeling error}}+\underbrace{\|u^h-u\|^2}_{\text{Discretization error}}).
\end{equation*}

Thus the generalization error is bounded above by three different terms, which are respectively the \textit{optimization error}, \textit{modeling error}, and \textit{discretization error}. The discretization error is only dependent on the choice of the discretization and the modeling error is based on the choice of the network architecture and hence both the errors remains static through out the optimization.

To facilitate the understanding of how $\|e^k_G\|^2$ converges and to simplify the analysis, we will next combine the optimization error and modeling error, while following standard analysis techniques from numerical analysis for the discretization error. To achieve that, the following assumption is required.

\begin{assumption}\label{assumption_1}
There exists a constant $c>0$ such that $\|u_{\theta^*}-u^h\|\leq c \|u_{\theta_k}-u_{\theta^*}\|, \forall k\geq 0$.
\end{assumption}
Assumption~\ref{assumption_1} implies that the modeling error can be upper bounded by the optimization error up to a constant. As the modeling error is determined by two categories of optima, this intuitively makes sense since we can always find some constant $c$ to let the above assumption hold. $c$ could be large when $u_{\theta_k}$ is close to $u_{\theta^*}$. While in practice an approximate solution for $u_{\theta^*}$ is typically sufficient such that $c<\infty$. Thus, we now have
\begin{equation}\tag{\ref{dynamic_gen_err_1}}
    \|e^k_G\|^2\leq 3[(1+c^2)\|u_{\theta_k}-u_{\theta^*}\|^2+\|u^h-u\|^2]. 
\end{equation}
Now the problem is to find the bounds for the two terms on the right hand side of \Eqnref{dynamic_gen_err_1}.

Since the solution function $u^h$ has been parameterized by $\theta$, to obtain the upper bound for the generalization error, we need to establish the relationship between $\|u_{\theta_k}-u_{\theta^*}\|$ and $\|\theta_k-\theta^*\|$. Recall the following equations:
\begin{subequations}
    \begin{alignat}{2}
        &u_{\theta_k}:=u^h(\mathbf{x};\theta_k)=\sum_{i=1}^N\mathcal{B}_i(\mathbf{x})U_i(\theta_k)\\
        &u^*:=u^h(\mathbf{x};\theta^*)=\sum_{i=1}^N\mathcal{B}_i(\mathbf{x})U_i(\theta^*)
    \end{alignat}
\end{subequations}
Using the last two equations yields
\begin{equation}\label{eq_6}
    \begin{split}
        \|u_{\theta_k}-u_{\theta^*}\|&=\|\sum_{i=1}^N\mathcal{B}_i(\mathbf{x})U_i(\theta_k)-\sum_{i=1}^N\mathcal{B}_i(\mathbf{x})U_i(\theta^*)\|\\&=\|\sum_{i=1}^N\mathcal{B}_i(\mathbf{x})(U_i(\theta_k)-U_i(\theta^*))\|\\&\leq \sum_{i=1}^N\|\mathcal{B}_i(\mathbf{x})U_i(\theta_k)-\mathcal{B}_i(\mathbf{x})U_i(\theta^*)\|\\&\leq\sum_{i=1}^N\|\mathcal{B}_i(\mathbf{x})\|\|U_i(\theta_k)-U_i(\theta^*)\|,
    \end{split}
\end{equation}
where the first inequality follows the triangle inequality and the second Cauthy-Schwartz inequality. As the nodal value function $U_i(\cdot)$ is parameterized by the weights $\theta$, in the optimization area, there exists generic assumption on its Lipschitz continuity. One may argue that this may not be practically feasible during implementation, but a regularization can be added to make it true. For simplicity, we impose it in the following directly:
\begin{assumption}\label{assumption_2}
There exists a constant $L>0$ such that for all $x, y\in\mathbb{R}^d$, $\|U_i(x)-U_i(y)\|\leq L\|x-y\|$, for all $i\in\{1,2,...,N\}$.
\end{assumption}
Hence, Eq.~\ref{eq_6} can be rewritten as
\begin{equation}
    \|u_{\theta_k}-u_{\theta^*}\|\leq L\leq\sum_{i=1}^N\|\mathcal{B}_i(\mathbf{x})\|\|\theta_k-\theta^*\|,
\end{equation}
where $\theta^*$ is the theoretical optimum of weights. Combining the last inequality with Eq.~\ref{dynamic_gen_err_1} yields
\begin{equation}
\begin{split}\label{eq_8}
        \|e^k_G\|^2&\leq 3[(1+c^2)L^2(\sum_{i=1}^N\|\mathcal{B}_i(\mathbf{x})\|\|\theta_k-\theta^*\|)^2 + \|u^h-u\|^2]\\&\leq 3[(1+c^2)L^2N\sum_{i=1}^N\|\mathcal{B}_i(\mathbf{x})\|^2\|\theta_k-\theta^*\|^2+\|u^h-u\|^2],
\end{split}
\end{equation}
where the second inequality follows from that $(\sum_{i=1}^N\|a_i\|)^2\leq N\sum^N_{i=1}\|a_i\|^2$. As the basis function $\mathcal{B}_i(\cdot)$ is a bounded function, we denote by $\hat{\mathcal{B}}$ the upper bound over all $i$. Hence, Eq.~\ref{eq_8} is rewritten as
\begin{equation}\label{eq_15}
    \|e^k_G\|^2\leq 3[(1+c^2)L^2N^2\hat{\mathcal{B}}^2\|\theta_k-\theta^*\|^2+\|u^h-u\|^2].
\end{equation}
\begin{remark}
We remark why the stochastic optimization theory can be applied in this context to study the convergence behavior of the \ibn{}. As will be presented in the latter sections, some assumptions will be imposed on the $\mathcal{J}$. However, in the PDE domain, such assumptions may not be practically feasible. It should be noted that though the application focused is PDE, the functions have been \textit{parameterized} such that the property of loss $\mathcal{J}$ is similar as in the machine learning (even deep learning)-based optimization problems. Additionally, the algorithm we leverage to search for the optimal weights are the popular gradient descent type. These motivate us to study the convergence rate in this work, following some similar assumptions adopted in the stochastic optimization literature.
\end{remark}
\begin{remark}
In this work, the analysis for the generalization error is based on the gradient descent type of algorithms. While there are other types of popular algorithms, such as accelerated or adaptive gradient descent, (stochastic) gradient descent remains the most fundamental one. Thus, this is the first step to facilitate the understanding of different errors in the PDE domain. Note, however that, for the implementation, other algorithms can be adopted, and we envision that the (stochastic) gradient descent algorithm may not perform the best empirically. Another aspect we would like to point out is that though due to different algorithms, the explicit error bounds vary, but the convergence rate can still be similar.
\end{remark}

\subsection{An Ideal Bound}
We first investigate the convergence rate when $\mathcal{J}$ is \textit{strongly convex and in a batch setting}. Via Eq.~\ref{eq:combined-loss-winding-num}, it can be seen that the two terms for boundary condition and exterior points are quadratic, which play a role as regularizers, turning the whole loss strongly convex if the first term of PDE loss is convex. While this scenario is typically not practical in the real implementation, it will show us the convergence behavior and an ideal bound that can deliver meaningful insights in terms of optimality and generalization error. To characterize the first main result, we present two well-known lemmas as follows.
\begin{lemma}\label{lemma_1}
If a continuously differentiable function $f:\mathbb{R}^d\to\mathbb{R}$ is $\mu$-strongly convex, for all $x, y\in\mathbb{R}^d$, then
\begin{equation}
    \langle\nabla f(x)-\nabla (y), x-y\rangle \geq \mu\|x-y\|^2.
\end{equation}
\end{lemma}

\begin{lemma}\label{lemma_2}
If a continuously differentiable function $f:\mathbb{R}^d\to\mathbb{R}$ is $\mu$-strongly convex and $\beta$-smooth, for all $x, y\in\mathbb{R}^d$, then
\begin{equation}
    \langle\nabla f(x)-\nabla (y), x-y\rangle \geq 
    \frac{\mu\beta}{\mu+\beta}\|x-y\|^2 + \frac{1}{\mu+\beta}\|\nabla f(x)-\nabla f(y)\|^2.
\end{equation}
\end{lemma}
With these two lemma in hand, we are now ready to present a key auxiliary lemma.
\begin{lemma}\label{lemma_3}
Suppose that $\mathcal{J}$ is $\mu$-strongly convex and $\beta$-smooth. By applying the gradient descent algorithm with a constant step size $\eta_k=\eta=\frac{2}{\mu+\beta}$, the iterates $\{\theta_k\}$ satisfy the following relationship
\begin{equation}\label{eq_18}
    \|\theta_K-\theta^*\|^2\leq \bigg(1-\frac{2}{\kappa+1}\bigg)^{2K}\|\theta_0-\theta^*\|^2,
\end{equation}
where $\kappa=\frac{\beta}{\mu}\geq 1$ and $K$ is the number of epochs.
\end{lemma}
\begin{proof}
Based on the gradient descent update, we have
\begin{equation}
\begin{split}
    \|\theta_k-\theta^*\|^2 &= \|\theta_{k-1}-\eta\nabla \mathcal{J}(\theta_{k-1})-\theta^*\|^2\\&=\|\theta_{k-1}-\theta^*\|^2 - 2\eta\langle\nabla \mathcal{J}(\theta_{k-1}), \theta_{k-1}-\theta^*\rangle+\eta^2\|\nabla \mathcal{J}(\theta_{k-1})\|^2\\&\leq \bigg(1-2\frac{\mu\beta\eta}{\mu+\beta}\bigg)\|\theta_{k-1}-\theta^*\|^2+\bigg(\eta^2-\frac{2\gamma}{\mu+\beta}\bigg)\|\nabla \mathcal{J}(\theta_{k-1})\|^2\\&=\bigg(1-\frac{2}{\kappa+1}\bigg)^2\|\theta_{k-1}-\theta^*\|^2\\&\leq\bigg(1-\frac{2}{\kappa+1}\bigg)^{2k}\|\theta_0-\theta^*\|^2.
\end{split}
\end{equation}
The first inequality follows from the substitution of the step size and Lemma~\ref{lemma_2}. While the third equality is due to $\eta^2-\frac{2\gamma}{\mu+\beta} = 0$. The desirable result is obtained by changing $k$ to $K$.
\end{proof}

Lemma~\ref{lemma_3} implies the upper bound of $\|\theta_K-\theta^*\|^2$ and shows the \textit{linear} convergence rate with a constant $\rho_1=\bigg(1-\frac{2}{\kappa+1}\bigg)^2$. Intuitively, the convergence to $\theta^*$ is quite fast based on the above result, particularly when the initialization $\theta_0$ is close to $\theta^*$. That said, if the loss induced by the \ibn{} is approximately strongly convex, and smooth, the optimization error decays in a linear rate in a batch setting, which is summarized in the following. Before that, we present another lemma that shows the upper bound of the discretization error $\|u^h-u\|^2$.
\begin{lemma}\label{lemma_4}
Assume that the basis function $\mathcal{B}_i(\mathbf{x})$ are chosen such that they are at least continuously differentiable locally over a mesh. Then the following relationship holds true:
\begin{equation}\label{eq_19}
    \|u^h-u\|^2\leq Ch^{2\alpha},
\end{equation}
where $C>0$ and $\alpha$ is the order of continuous derivative. Typically, $\alpha\geq 1$.
\end{lemma}
The proof follows an analysis similar to that presented in ~\citep{ern2021finite,main2018shifted,hansbo2002unfitted}. Lemma~\ref{lemma_4} states the static property of the discretization error, and implies that the decaying of the error is bounded by the second order of the resolution $h$. This observation is critical as the number of basis functions $N$ has a relationship with $h$ such that a tradeoff between the optimization error and discretization error can be found. More detail will be shown in the latter analysis. We are now ready to present the first main result.
\begin{theorem}\label{theorem_1}
Suppose that $\mathcal{J}$ is $\mu$-strongly convex and $\beta$-smooth. Let Assumptions~\ref{assumption_1} and~\ref{assumption_2} hold. By appling the gradient descent algorithm with a constant step size $\eta_k=\eta=\frac{2}{\mu+\beta}$, the generalization error $\|e^K_G\|^2$ satisfies the following relationship after $K$ epochs,
\begin{equation}
    \|e^K_G\|^2\leq 3(1+c^2)L^2N^2\hat{\mathcal{B}}^2\bigg(1-\frac{2}{\kappa+1}\bigg)^{2K}\|\theta_0-\theta^*\|^2+3Ch^{2\alpha}.
\end{equation}
\end{theorem}
\begin{proof}
The desirable result can immediately be obtained by combining Eqs.~\ref{eq_15},~\ref{eq_18}, and~\ref{eq_19}.
\end{proof}
Theorem~\ref{theorem_1} implies that the generalization error of \ibn{} is upper bounded by two terms, while the first one is tightly related to the optimization error caused by numerically searching for the optimal weights and the second one shows the static error due to the discretization. Additionally, in an asymptotic manner, we have that
\begin{equation}
    \lim_{K\to\infty}\|e^K_G\|^2\leq 3Ch^{2\alpha},
\end{equation}
which suggests that the generalization error will ultimately dominated by the discretization error determined by the resolution $h$ and the order $\alpha$, both of which rely on the definition of basis functions $\{\mathcal{B}_i(\mathbf{x})\}_{i=1}^N$. According to the domain knowledge~\citep{larson2013finite}, the resolution $h$ is a function of the number of basis functions, resulting in the $h\propto\frac{1}{\sqrt{N}}$ (due to the 2D analysis in this context). With this relationship, if $h$ is chosen such that $h=\mathcal{O}(\frac{1}{\sqrt{KN}})$, the following corollary can be obtained:
\begin{corollary}\label{coro_1}
Suppose that $h=\mathcal{O}(\frac{1}{\sqrt{KN}})$. Given all conditions and parameters from Theorem~\ref{theorem_1}, the generalization error achieves an overall \textit{sublinear} convergence rate
\begin{equation}
    \|e^K_G\|^2\leq\mathcal{O}\bigg(N^2\rho_1^K+\frac{1}{(NK)^{\alpha}}\bigg).
\end{equation}
\end{corollary}
A quick observation is that the generalization error depends primarily on the number of basis functions $N$ and the number of epochs $K$. Though Corollary~\ref{coro_1} implies a satisfactory convergence rate, it requires strong conditions such as batch setting and strongly convex loss. While, surprisingly, with an explicit relationship between $h$ and $N$ in the error bound, it shows a clear \textit{tradeoff} between the optimization error and discretization error such that a smaller $N$ can leads to an easier optimization procedure but causing larger discretization error. Such a result can also help the practitioners in the practical design for selecting the (nearly) optimal $N$ or $h$. In the following, we will make attempts to relax the conditions for either the setting or loss function and verify if such a tradeoff is still existing.
\subsection{Stochastic Setting \& Polyak-Lojasiewicz Condition}
In a batch setting, the computational overhead could be an issue when the number of data points is significantly large. Thus, stochastic gradient descent (SGD)~\citep{bottou2012stochastic} has been one of the most popular optimization algorithms. Further, the loss functions associated with most real-world problems are not necessarily strongly convex. In Eq.~\ref{eq:combined-loss-winding-num}, the loss consists of three terms and the last two terms are strongly convex, while the property of the first term can determine the property of the whole loss. In this section, we present the theoretical analysis for the generalization error $e^k_G$ with the stochastic setting, and the loss satisfies a relaxed condition, termed \textit{Polyak-Lojasiewicz (PL) condition}, which complies more with the practical implementation and has been provable shown to hold across most of the parameter space in over-parameterized networks~\citep{liu2022loss}. It has been acknowledged that the strong convexity implies PL condition, but not vice versa, as the loss can possibly be \textit{non-convex}.
When applying SGD, the sampling technique is adopted for calculating the gradients such that they are \textit{noisy}. Popular random sampling techniques such as uniform sampling $\mathcal{U}(\{1, ..., N_s\})$ have been used widely. This motivates us to define a $\sigma$-field for each time step $k$, i.e., $\mathcal{F}_k=\sigma(I_0,...,I_k)$, where $\mathcal{F}_{-1}=\{\Omega, \emptyset\}$, and $I_k\sim\mathcal{U}(\{1, ..., N_s\})$ indicates each realization of random sampling. It is noted that $\theta_k$ is $\mathcal{F}_{k-1}$-measurable, i.e., $\theta_k$ only depends on $I_0,...,I_{k-1}$. Based on the defined setting here, one well-known result is that a stochastic gradient is an independent and unbiased estimate of the gradient $\mathbb{E}[\nabla\mathcal{J}_{I_k}(\theta_k)|\mathcal{F}_{k-1}]=\nabla \mathcal{J}(\theta_k)$. Due to the stochasticity, Eq.~\ref{eq_15} is now rewritten by taking expectation on both sides
\begin{equation}\label{eq_25}
    \mathbb{E}[\|e^k_G\|^2|\mathcal{F}_{k-1}]\leq 3[(1+c^2)L^2N^2\hat{\mathcal{B}}^2\mathbb{E}[\|\theta_k-\theta^*\|^2|\mathcal{F}_{k-1}]+\|u^h-u\|^2].
\end{equation}
Similarly, to obtain the upper bound for $\mathbb{E}[\|e^k_G\|^2|\mathcal{F}_{k-1}]$, the key is to investigate the upper bound of $\mathbb{E}[\|\theta_k-\theta^*\|^2|\mathcal{F}_{k-1}]$, which is shown in the next auxiliary lemma, adapted from~\citep{karimi2016linear}. We first present the PL inequality in the following.
\begin{equation}
    \frac{1}{2}\|\nabla \mathcal{J}(\theta)\|^2\geq \mu(\mathcal{J}(\theta)-\mathcal{J}(\theta^*)), \forall \theta.
\end{equation}
Intuitively, this inequality suggests that the gradient grows faster than a quadratic function as it moves away from the optimal function value. Also, every stationary point is a global optimum when the loss satisfies this condition. This differs from the strong convexity, which implies a unique solution.
\begin{lemma}\label{lemma_5}
Let $(\mathcal{F}_k)_k$ be an increasing family of $\sigma$-fields. For each $k\geq 0$, suppose that $\mathcal{J}$ is a $\beta$-smooth, differentiable, and $\mathcal{F}_k$-measurable function over $\Theta$. Also, assume that $\mathbb{E}[\nabla\mathcal{J}_{I_k}(\theta_k)|\mathcal{F}_{k-1}]=\nabla \mathcal{J}(\theta_k)$, for each $\theta\in\Theta$ and $k\geq 1$, where $\mathcal{J}$ satisfies the $\mu$-PL condition. If $\forall k\geq 0$, $\mathbb{E}[\|\nabla \mathcal{J}_{I_k}(\theta_k)\|^2|\mathcal{F}_{k-1}]\leq \sigma^2$, then the iterates $\{\theta_k\}$ satisfy the following relationship after $K$ epochs,
\begin{equation}\label{eq_27}
    \mathbb{E}[\|\theta_K-\theta^*\|^2]\leq (1-2\mu\eta)^{K-1}\frac{2[\mathcal{J}(\theta_0)-\mathcal{J}(\theta^*)]}{\mu}+\frac{\beta\sigma^2\eta}{2\mu^2},
\end{equation}
when the step size $\eta_k=\eta<\frac{1}{2\mu}$.
\end{lemma}
Based on the update from the SGD, and the definition of $\beta$-smooth, we have the following relationship
\begin{equation}
    \mathcal{J}(\theta_{k+1}) \leq \mathcal{J}(\theta_{k})-\eta_k\langle\nabla \mathcal{J}(\theta_{k}), \nabla \mathcal{J}_{I_{k+1}}(\theta_{k})\rangle + \frac{\beta\eta_k^2}{2}\|\nabla \mathcal{J}_{I_{k+1}}(\theta_{k})\|^2.
\end{equation}
Taking the expectation for both sides w.r.t. $I_{k+1}$ yields
\begin{equation}
\begin{split}
    \mathbb{E}[\mathcal{J}(\theta_{k+1})|\mathcal{F}_k]&\leq \mathcal{J}(\theta_{k})-\eta_k\langle\nabla \mathcal{J}(\theta_{k}), \mathbb{E}[\nabla \mathcal{J}_{I_{k+1}}(\theta_{k})|\mathcal{F}_k]\rangle+\frac{\beta\eta_k^2}{2}\mathbb{E}[\|\nabla \mathcal{J}_{I_{k+1}}(\theta_{k})\|^2|\mathcal{F}_k]\\&\leq \mathcal{J}(\theta_{k})-\eta_k\|\nabla \mathcal{J}(\theta_{k})\|^2+\frac{\beta\sigma^2\eta^2_k}{2}\\&\leq \mathcal{J}(\theta_{k})-2\mu\eta_k(\mathcal{J}(\theta_{k})-\mathcal{J}(\theta^*))+\frac{\beta\sigma^2\eta^2_k}{2},
\end{split}
\end{equation}
where the second inequality follows from the unbiased estimate of a stochastic gradient and the bounded second moment, the third inequality follows from the PL condition. Subtracting $\mathcal{J}(\theta^*))$ on both sides leads to the following inequality
\begin{equation}\label{eq_36}
    \mathbb{E}[\mathcal{J}(\theta_{k+1})-\mathcal{J}(\theta^*)|\mathcal{F}_k]\leq (1-2\mu\eta_k)(\mathcal{J}(\theta_{k})-\mathcal{J}(\theta^*))+\frac{\beta\sigma^2\eta^2_k}{2}.
\end{equation}
As $\eta_k=\eta<\frac{1}{2\mu}$, by applying the last inequality recursively, the following can be obtained
\begin{equation}
\begin{split}
    \mathbb{E}[\mathcal{J}(\theta_{k+1})-\mathcal{J}(\theta^*)|\mathcal{F}_k]&\leq(1-2\mu\eta)^k(\mathcal{J}(\theta_{0})-\mathcal{J}(\theta^*))+\frac{\beta\sigma^2\eta^2}{2}\sum_{j=0}^k(1-2\mu\eta)^j\\&\leq(1-2\mu\eta)^k(\mathcal{J}(\theta_{0})-\mathcal{J}(\theta^*))+\frac{\beta\sigma^2\eta^2}{2}\sum_{j=0}^\infty(1-2\mu\eta)^j\\&\leq (1-2\mu\eta)^k(\mathcal{J}(\theta_{0})-\mathcal{J}(\theta^*))+\frac{\beta\sigma^2\eta}{4\mu}.
\end{split}
\end{equation}
The last inequality follows the definition of step size and the property of the geometric series.
As the PL condition implies the \textit{quadratic growth} property~\citep{karimi2016linear}, which suggests \begin{equation}\label{eq_38}
    \mathcal{J}(\theta)-\mathcal{J}(\theta^*)\geq \frac{\mu}{2}\|\theta-\theta^*\|^2.
\end{equation}
Thus, combining the last two inequalities and replacing $k$ with $K$ completes the proof.

Lemma~\ref{lemma_5} shows that with a constant step size $\eta$, $\theta_k$ will converge linearly to the neighborhood of $\theta^*$, up to a constant w.r.t. $\eta$ and $\sigma^2$. This indicates the impact of the noisy stochastic gradient, while we can reduce it by selecting a smaller step size, which, however, would increase $1-2\mu\eta$, resulting in slower convergence.

It is ready to state the following main result for the generalization error in the stochastic setting with the PL condition.
\begin{theorem}\label{theorem_2}
Given all conditions and parameters defined in Lemma~\ref{lemma_5} and let Assumptions~\ref{assumption_1} and~\ref{assumption_2} hold. Then the generalization error $\mathbb{E}[\|e_G^K\|^2]$ satisfies the following relationship after $K$ epochs,
\begin{equation}
        \mathbb{E}[\|e^K_G\|^2]\leq 3(1+c^2)L^2N^2\hat{\mathcal{B}}^2\bigg((1-2\mu\eta)^{K-1}\frac{2[\mathcal{J}(\theta_0)-\mathcal{J}(\theta^*)]}{\mu}+\frac{\beta\sigma^2\eta}{2\mu^2}\bigg)+3Ch^{2\alpha}.
\end{equation}
\end{theorem}
\begin{proof}
Combining Eq.~\ref{eq_25}, Lemma~\ref{lemma_4} and Lemma~\ref{lemma_5} immediately yields the above result.
\end{proof}
We now analyze the error bound in this context. Inherently, the optimization error is directly impacted by the updates of weight parameters from Lemma~\ref{lemma_5} such that in an asymptotic manner, the generalization error converges to a constant involving both optimization error and discretization error eventually, i.e., \[\lim_{K\to\infty}\mathbb{E}[\|e_G^K\|^2]\leq 3(1+c^2)L^2N^2\hat{\mathcal{B}}^2\frac{\beta\sigma^2\eta}{2\mu^2}+3Ch^{2\alpha}.\] Compared to Theorem~\ref{theorem_1}, the asymptotic error is larger due to the stochastic gradient noise term $\frac{\beta\sigma^2\eta}{2\mu^2}$, which can be enlarged by the number of basis functions $N$. This may suggest a relative worse minimum, which has also been shown in previous works~\citep{gower2019sgd,bottou2018optimization}. As discussed above, one can leverage the step size to control its negative impact, while this also affects the convergence of the optimization procedure. The discretization error remains the same because of its static nature. Similarly, when we apply the same relationship between $N$ and $h$, the following corollary summarizes the result in a noisy environment.
\begin{corollary}\label{coro_2}
Suppose that $h=\mathcal{O}(\frac{1}{\sqrt{KN}})$. Given all conditions and parameters defined in Theorem~\ref{theorem_2}, the generalization error enjoys a sublinear convergence rate and converges to a neighborhood of the minimum up to a constant w.r.t. $N^2\sigma^2$,
\begin{equation}
    \mathbb{E}[\|e_G^K\|^2]\leq \mathcal{O}\bigg(N^2\rho_2^K+N^2\sigma^2+\frac{1}{(NK)^{\alpha}}\bigg),
\end{equation}
where $\rho_2 = (1-2\mu\eta)$.
\end{corollary}
Similar to Corollary~\ref{coro_1}, an immediate observation from Corollary~\ref{coro_2} is that the error bound still involves the clear tradeoff between the optimization error and discretization error. However, a surprising difference is that the optimization error dominates the generalization error instead of discretization error in Corollary~\ref{coro_1} when $K$ is sufficiently large, as the constant term $N^2\sigma^2$ exists. Additionally, this result may imply a relatively higher overfitting than that in the batching setting due to larger optimization error. Apparently, the careful design of $N$ is able to alleviate the overfitting phenomenon.
To reduce the uncertainties brought by the noisy gradients, a decay step size needs to be adopted, instead of a constant one, while this results in a worse sublinear convergence rate for the optimization error itself. Hence, we select a decaying step size for the eliminate the effect of the noisy gradient in the optimization error.

Let the step size $\eta_k = \frac{2k+1}{2\mu(k+1)^2}$. The following lemma describes how the weight parameters evolve accordingly.
\begin{lemma}\label{lemma_6}
Let $(\mathcal{F}_k)_k$ be an increasing family of $\sigma$-fields. For each $k\geq 0$, suppose that $\mathcal{J}$ is a $\beta$-smooth, differentiable, and $\mathcal{F}_k$-measurable function over $\Theta$. Also, assume that $\mathbb{E}[\nabla\mathcal{J}_{I_k}(\theta_k)|\mathcal{F}_{k-1}]=\nabla \mathcal{J}(\theta_k)$, for each $\theta\in\Theta$ and $k\geq 1$, where $\mathcal{J}$ satisfies the $\mu$-PL condition. If $\forall k\geq 0$, $\mathbb{E}[\|\nabla \mathcal{J}_{I_k}(\theta_k)\|^2|\mathcal{F}_{k-1}]\leq \sigma^2$, then the iterates $\{\theta_k\}$ satisfy the following relationship after $K$ epochs,
\begin{equation}\label{eq_30}
    \mathbb{E}[\|\theta_K-\theta^*\|^2]\leq \frac{\beta\sigma^2}{K\mu^3},
\end{equation}
when the step size $\eta_k=\frac{2k+1}{2\mu(k+1)^2}$.
\end{lemma}
\begin{proof}
Recalling Eq.~\ref{eq_36} and applying the step size $\eta_k=\frac{2k+1}{2\mu(k+1)^2}$ to it results in
\begin{equation}
    \mathbb{E}[\mathcal{J}(\theta_{k+1})-\mathcal{J}(\theta^*)|\mathcal{F}_k]\leq\frac{k^2}{(k+1)^2}(\mathcal{J}(\theta_{k})-\mathcal{J}(\theta^*))+\frac{\beta\sigma^2(2k+1)^2}{8\mu^2(k+1)^4}.
\end{equation}
We now multiply both sides by $(k+1)^2$ and denote $\zeta_k:=k^2\mathbb{E}[\mathcal{J}(\theta_{k})-\mathcal{J}(\theta^*)]$ such that
\begin{equation}
    \zeta_{k+1}\leq \zeta_k + \frac{\beta\sigma^2(2k+1)^2}{8\mu^2(k+1)^4}\leq \zeta_k +\frac{\beta\sigma^2}{2\mu^2},
\end{equation}
where the second inequality uses the fact that $\frac{2k+1}{k+1}<2$. Summing up the last inequality from 0 to $k$ and using the fact that $\zeta_0=0$ we can obtain the following relationship
\begin{equation}
    \zeta_{k+1}\leq\zeta_0+(k+1)\frac{\beta\sigma^2}{2\mu^2}= (k+1)\frac{\beta\sigma^2}{2\mu^2}.
\end{equation}
Applying the last inequality to $K$ and adopting the quadratic growth property in Eq.~\ref{eq_38} completes the proof.
\end{proof}

When $K\to\infty$, a diminishing step size will eliminate the negative impact of $\sigma^2$ to enable the convergence to $\theta^*$, while at cost of the convergence rate. If defining a condition number as $\frac{\beta}{\mu}$, we can observe that a smaller condition number enables a better solution given a finite $K$ in practice. With this in hand, the generalization error bound is obtained as follows.
\begin{theorem}\label{theorem_3}
Given all conditions and parameters defined in Lemma~\ref{lemma_6} and let Assumptions~\ref{assumption_1} and~\ref{assumption_2} hold. Then, the generalization error $\mathbb{E}[\|e^K_G\|^2]$ satisfies the following relationship after $K$ epochs,
\begin{equation}
    \mathbb{E}[\|e^K_G\|^2]\leq 3(1+c^2)L^2N^2\hat{\mathcal{B}}^2\frac{\beta\sigma^2}{K\mu^3}+3Ch^{2\alpha}.
\end{equation}
\end{theorem}
\begin{proof}
The proof follows from the combination among Eq.~\ref{eq_25}, Lemma~\ref{lemma_4} and Lemma~\ref{lemma_6}.
\end{proof}
Theorem~\ref{theorem_3} renders a similar asymptotic behavior as seen in Theorem~\ref{theorem_1}, that said, the generalization error will end up with the discretization error when $K\to\infty$, but with a slower convergence rate $\mathcal{O}(\frac{1}{K})$, regardless of the noise brought by stochastic gradients. Now the following corollary summarizes the convergence rate given the explicit relationship between $h$ and $N$. 
\begin{corollary}\label{coro_3}
Suppose that $h=\mathcal{O}(\frac{1}{\sqrt{KN}})$. Given all conditions and parameters defined in Theorem~\ref{theorem_3}, the generalization error enjoys a sublinear convergence rate as follows
\begin{equation}
    \mathbb{E}[\|e^G_K\|^2]\leq\mathcal{O}\bigg(\frac{N^2}{K}+\frac{1}{(NK)^{\alpha}}\bigg).
\end{equation}
\end{corollary}
Corollary~\ref{coro_3} still enjoys the tradeoff between the optimization error and discretization error, while the rate is determined primarily by the former, given a sufficiently large $K$. This observation is different from what we have found in Corollary~\ref{coro_1}, where no noise in gradient needs to be taken into account. Hence, dealing with noisy gradients motivates us to leverage different step sizes in various scenarios. Additionally, when $\alpha=1$, the convergence rate is $\mathcal{O}(1/K)$.
\begin{remark}
To summarize, through the investigations on the generalization error by using the optimization theory, we have clearly found that there exists tradeoff between optimization error and discretization error and that depending on different scenarios, either of them determines the primary convergence rate, after a sufficiently large number of epochs and given a reasonable explicit relationship among $h, N$ and $K$. This initially fills in the gap between theory and practice in the parametric PDE domain. Though the properties of the loss in this work may not necessarily be practical in the real-world problem, our analysis delivers some useful and meaningful theoretical insights to the community and points out a research direction that other researchers can work on. Perhaps the more generic case for the loss is non-convex without the PL condition, while we leave this as one of our future work as the accuracy metric for the weights could be significantly from what we have used here, i.e., $\|\theta_k-\theta^*\|^2$, and it can definitely affect the definition of the generalization error.
\end{remark}





\section{Additional Implementation Details}

\subsection*{Network Architecture}
The neural network architecture $ G $ mainly comprises 2D and 3D transpose convolution operations. The dimensionality of the transpose convolution operation will match the dimensionality of the domain we are solving for. We use transpose convolutions as a trainable up-sampling scheme. This is favored compared to deterministic methods such as interpolation-based upsampling methods or unpooling operations due to the added learning capacity transpose convolutions offer. The kernel size of the transpose convolution operation is initialized to encapsulate information from a neighborhood of NURBS control points. 

\subsection*{Application of Boundary Conditions}

We can accurately apply the boundary conditions to arbitrary complex domains. We accomplish this by learning the in-out function associated with a given boundary. We optimize our network to learn the in-out function using a differentiable winding numbers loss function outlined in detail below. The optimized in-out function approximation sets the elements inside the boundary to 1 and outside the boundary to 0. The inside of the object also inherits the conditions at its boundary; therefore, the Dirichlet condition is applied to all the points in the background mesh assigned a positive winding number, along with the points on the object surface. The nodes assigned zero values are understood to belong to the computational domain, thus being used for further computations. This is accomplished with a \texttt{torch.where()} function. We then apply this processed in-out prediction as the boundary condition to the predicted field solution using a similar \texttt{torch.where()} function. 

\subsection*{Differentiation and Integration:}

The computation of the loss function $ \mathcal{J} $  in \Eqnref{eq:combined-loss-winding-num} requires an integration of the term $ v^h[\mathcal{N}(w^h)-f] $ over $ \Omega_B $. And the evaluation of $ \mathcal{N}(w^h) $ in turn requires some differentiation (recall that $ \mathcal{N}(\cdot) $ is a differential operator; also compare \Eqnref{eq:poisson-intro} and \Eqnref{eq:ns-intro}). But, this differentiation is not to be confused with the differentiation through the neural network. Rather, this derivative calculation is done using the basis functions $ \{\mathcal{B}_i(\mvec{x})\}_{i=1}^{N} $ and therefore the neural network is not responsible this.

Finally, the integration is performed numerically using Gaussian quadratures. For a given mesh $ \mesh $, the basis functions $ \{\mathcal{B}_i\} $ as well as the quadrature points are known and are completely deterministic. This allows us to define the spatial gradients of the predicted field solution by simply evaluating each element with the first or second-order derivative of the basis function originally used to evaluate the predicted field solution. Since we use the FEM, we perform the integration in each discrete element and then perform a summation over the finite set of elements to obtain the total integral.

\subsection*{Dataset Generation:}

In this work we used a dataset consisting of point clouds, and the corresponding normals, which outline Non-Uniform Rational B-Splines (NURBS) curves to represent the boundaries of irregular domains. 
NURBS curves are presented by a set of control points, with each control point being described by cartesian coordinates, in this case only (\emph{x},\emph{y}). 
For this dataset, we selected points along the x-axis which were uniformly spaced from 0 to 1. The corresponding coordinates along the y-axis were randomly sampled from a uniform distribution with a minimum value of 0.2 and a maximum value of 0.8. To attain the point cloud on the boundary defined by the NURBS curve we utilized NURBS-Python\citep{bingol2019geomdl}, a geometric modeling library. NURBS-Python not only provides a point cloud on the boundary of the NURBS curve, but also the normals, unique vectors for each point in the point cloud pointing in the orthogonal direction with respect to the boundary. Additionally, the area for each point is required for the calculation of the winding number. In this work we assume each point has a uniform area, which we maintain for each irregular boundary in the entire dataset. 

\subsection*{Differentiable Winding Number Computation:}

As explained in the main paper, winding number for a given point cloud $\mathcal{P}$ at a given query point $\mathbf{q}$ can be computed as 
\begin{equation}
\chi_{w}(\mathbf{q}) = \sum_{i=1} ^m a_i \frac{(\mathbf{p}_i - \mathbf{q}) \cdot \hat{\mathbf{n}}_i}{4\pi || \mathbf{p}_i - \mathbf{q} ||^3}
\end{equation}
In order to evaluate the winding number at all the nodal locations, we perform all the pairwise distance computations between every nodal location and every point in the pointcloud and perform a sum reduction of the pairwise distances as per the above equation. Using \texttt{pytorch}, we can achieve this with a simple broadcasting operation. This way, all the operations are performed using the GPUs very quickly.

During experimentation, we demonstrated two different methods of generating masks to impose the boundary conditions on our predicted field solution during the Finite Element loss computation. The first was with a differentiable winding number method. The second leverages the Finite Element framework to solve the Eikonal equation. The field solution to the Eikonal equation is a signed distance field. Once a signed distance field is obtained for the irregular boundary, we use a differentiable conditional function to create an accurate binary mask from the signed distance field.

\subsection*{Poisson's Equation with $\chi_w$-based Geometry Loss:}
The canonical Poisson's equation was introduced in \Eqref{eq:poisson-intro}. In this section we apply IBN to solve the Poisson's equation with $ f = 0 $ and the boundary conditions defined as:
$ u = 1, \ \ \text{on} \ \ \Gamma_o $ and $ u = 0, \ \ \text{on} \ \ \Gamma_B.$
This case mimics a steady heat transfer problem where $ \Omega_o $ represents a heat source (of infinite capacity) and the outer boundary $ \Gamma_B $ represents a sink (of infinite capacity). The solution to \Eqnref{eq:poisson-intro} along with the above $ f $ and the boundary conditions is the temperature field under those conditions. This case essentially refers to a large family of different problems where $ \Gamma_B $ can be any closed surface such that $ \Omega_o\subset\Omega_B $. As discussed in \Secref{sec:methodology}, we consider the cases where such source object shapes are available as point cloud data, and our goal is to devise a fast neural method to use the point cloud data directly in analysis without having to go through a mesh generation process. \Figref{fig:imdiff-poisson-winding-test1} shows the results of the winding number and the corresponding solution to the PDE. 

\begin{figure}[t!]
	\centering
	\begin{subfigure}[b]{0.39\linewidth}
    \includegraphics[trim=0 0 0 0,clip,width=0.99\linewidth]{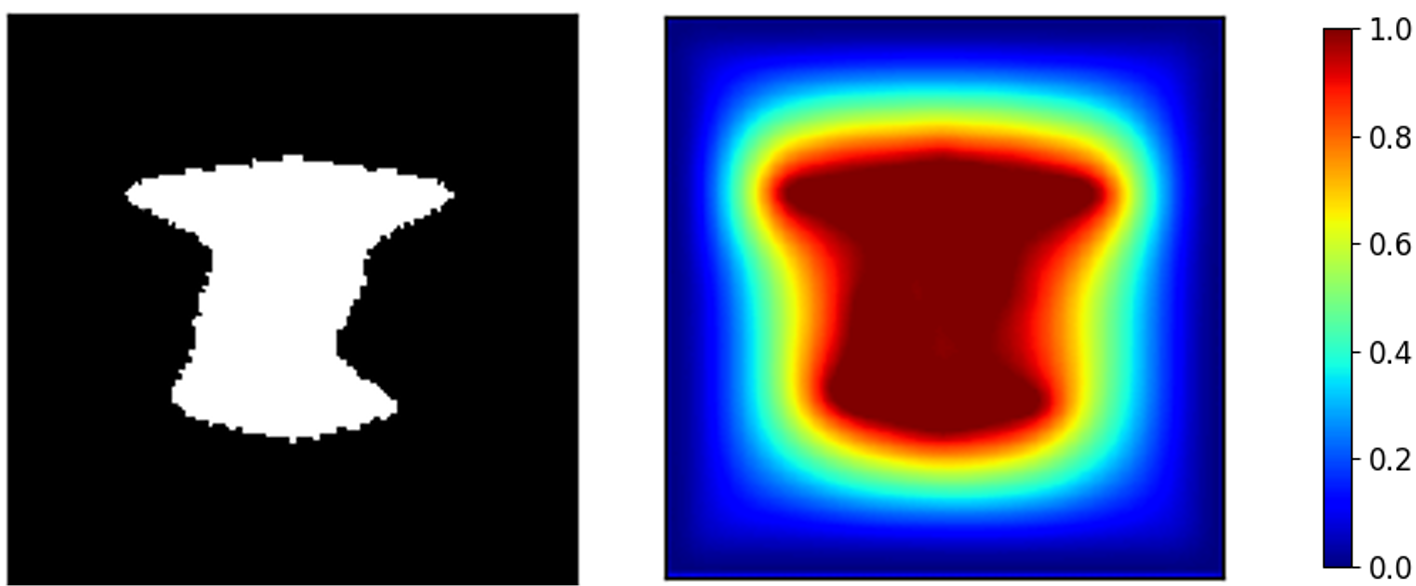}
	\caption{$\chi_w$-based Geometry Loss}
	\label{fig:imdiff-poisson-winding-test1}
	\end{subfigure}
	\begin{subfigure}[b]{0.60\linewidth}
	\centering
	\includegraphics[trim={0.2in 0.0in 0.2in 0.0in},clip,width=0.99\linewidth]{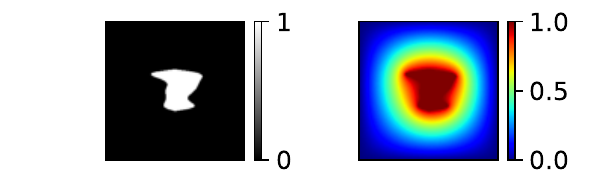}
	\caption{$\chi_e$-based Geometry Loss}
	\label{fig:imdiff-poisson-winding-test2}
	\end{subfigure}
	\caption{(a): Left: The in-out information is obtained by the winding number calculation. Right:~Poisson's equation is solved with an immersed object represented using a 2D point cloud. (b)~The in-out information is obtained by solving the Eikonal equation in addition to the solution to the Poisson's equation using the IBN framework.}
	\vspace{-0.1in}
\end{figure}

\subsection*{Poisson's Equation with $\chi_e$-based Geometry Loss:}
In this case, along with the Poisson's Equation, we solve Eikonal equation $(1+\tau)\|\grad \phi\| - \tau (\nabla \phi) = 1$, $\tau \in [0,0.5]$ with $\phi=0$ at $\Gamma_o$. For this case, we simultaneously solve for both the PDEs and show the results in \Figref{fig:imdiff-poisson-winding-test2}.

\end{document}